\documentclass{article}

\usepackage{microtype}
\usepackage{graphicx}
\usepackage{subfigure}
\usepackage{booktabs} % for professional tables

\usepackage{hyperref}

\usepackage[accepted]{icml2024}

\usepackage{amsmath}
\usepackage{amssymb}
\usepackage{mathtools}
\usepackage{amsthm}
\usepackage{upgreek}

\usepackage[capitalize,noabbrev]{cleveref}
\usepackage{autonum}

\theoremstyle{plain}
\newtheorem{theorem}{Theorem}[section]
\newtheorem{proposition}[theorem]{Proposition}

\theoremstyle{definition}

\theoremstyle{remark}

\newcommand{\final}[1]{{#1}}

\newcommand{\rmd}{\mathrm{d}}
\newcommand{\MMD}{\mathrm{MMD}}
\newcommand{\DMMD}{\mathrm{DMMD}}
\newcommand{\DKALE}{\mathrm{DKALE}}
\newcommand{\KALE}{\mathrm{KALE}}

\newcommand{\rset}{\mathbb{R}}
\newcommand{\nset}{\mathbb{N}}

\newcommand{\calL}{\mathcal{L}}
\newcommand{\kbase}{k_{\text{base}}}
\newcommand{\Niter}{N_{\text{iter}}}
\newcommand{\Nnoise}{N_{\text{noise}}}
\newcommand{\Nfixed}{N_{\text{s}}}
\newcommand{\Npart}{N_{\text{p}}}
\newcommand{\Nclean}{N_{\text{c}}}
\newcommand{\tmax}{t_{\text{max}}}
\newcommand{\tmin}{t_{\text{min}}}

\newcommand{\calLtot}{\calL_{\text{tot}}}
\newcommand{\Id}{\mathrm{Id}}
\newcommand{\argmax}{\mathrm{argmax}}
\newcommand{\mll}{<\!\!\!<}

\usepackage[textsize=tiny]{todonotes}

\icmltitlerunning{Deep MMD Gradient Flow without adversarial training}

\begin{document}

\twocolumn[
\icmltitle{Deep MMD Gradient Flow without adversarial training}

% unless you've provided the [accepted] option to the icml024


\icmlsetsymbol{equal}{*}

\begin{icmlauthorlist}
\icmlauthor{Alexandre Galashov}{google,ucl}
\icmlauthor{Valentin de Bortoli}{google}
\icmlauthor{Arthur Gretton}{google,ucl}
% \icmlauthor{Firstname3 Lastname3}{comp}
\end{icmlauthorlist}

\icmlaffiliation{google}{Google DeepMind}
\icmlaffiliation{ucl}{Gatsby Computational Neuroscience Unit, UCL}

\icmlcorrespondingauthor{Alexandre Galashov}{agalashov@google.com}

\icmlkeywords{Machine Learning, ICML}

\vskip 0.3in
]



\printAffiliationsAndNotice{} % otherwise use the standard text.

\begin{abstract}

%AG: generally I'd remove citations from the abstract besides in special circumstances.

We propose a gradient flow procedure for generative modeling by transporting particles from an initial source distribution to a target distribution, where the gradient field on the particles is given by a noise-adaptive \final{Wasserstein Gradient of the} Maximum Mean Discrepancy (MMD). 
The noise-adaptive MMD is trained on data distributions corrupted by increasing levels of noise, obtained via a forward diffusion process, as commonly used in denoising diffusion probabilistic models. %~\citep{ho2020denoising}.
The result is a generalization of  MMD Gradient Flow, which we call
Diffusion-MMD-Gradient Flow or $\DMMD$.
The divergence training procedure is related to discriminator training in Generative Adversarial Networks (GAN), but does not require adversarial training.
We obtain competitive empirical performance in unconditional image generation on CIFAR10, MNIST, CELEB-A (64 x64) and LSUN Church (64 x 64). Furthermore, we demonstrate the validity of the approach when MMD is replaced by a lower bound on the KL divergence.

\iffalse
We propose a novel approach for generative modeling based on noise adaptive version of Maximum Mean Discrepancy (MMD)~\citep{gretton12a} Gradient Flow~\citep{arbel2019maximum}. At the basis of the method is a noise-conditional MMD discriminator trained to distinguish clean from noisy data for a given level of noise. The noisy data is produced from the forward diffusion process commonly used in denoising diffusion probabilistic models~\citep{ho2020denoising}. The proposed procedure mimics Generative Adversarial Networks (GAN) training but does not require adversarial training. At inference time, the trained noise conditional discriminator is used in a noise-adaptive variant of MMD Gradient Flow. We refer to the training and inference procedures as Diffusion-MMD-gradient flow or \DMMD. We demonstrate competitive empirical performance of the method in unconditional image generation on CIFAR10 dataset. Moreover, we provide theoretical justifications behind the use of this noise-adaptive procedure in the MMD Gradient flow.
\fi
\end{abstract}

\section{Introduction}
\label{sec:introduction}

In  recent years, generative models have achieved impressive capabilities on image \cite{saharia2022photorealistic}, audio \cite{le2023voicebox} and video generation \cite{ho2022imagen} tasks but also protein modeling \cite{watson2022broadly} and 3d generation \cite{poole2022dreamfusion}. Diffusion models \citep{sohldickstein2015deep,ho2020denoising,song2020score,rombach2022highresolution} underpin these new methods. In these models, we learn a backward denoising diffusion process via denoising score matching~\citep{hyvarinen05a,vincent11}. This backward process corresponds to the time-reversal of a forward noising process. At sampling time, starting from random Gaussian noise, diffusion models produce samples by discretizing the backward process. 

One challenge that arises when applying these models in practice is that the Stein score (that is, the gradient log of the current noisy density) becomes ill-behaved near the data distribution~\citep{yang2023eliminating}: the diffusion process needs to be slowed down at this point, which incurs a large number of sampling steps near the data distribution. Indeed, if the manifold hypothesis holds \cite{tenenbaum2000global,fefferman2016testing,brown2022union} and the data is supported on a lower dimensional space, it is expected that the score will explode for noise levels close to zero, to ensure that the backward process concentrates on this lower dimensional manifold \cite{debortoli2023convergence,pidstrigach2022score,chen2022sampling}.  
While strategies exist to mitigate these issues, they  trade-off the quality of the output against inference speed, see for instance~\citep{song2023consistency,xu2023ufogen,sauer2023adversarial}. 
%This puts diffusion models at one end of the spectrum with high quality of samples but long generation time. 

Generative Adversarial Networks (GANs)~\citep{goodfellow2014generative} represent an alternative popular generative modelling framework~\citep{brock2019large,karras2020training}.
Candidate samples are produced by a {\em generator}: a neural net mapping low dimensional noise to high
dimensional images. The generator is trained in alternation with a {\em discriminator}, which is a measure of discrepancy
between the generator and target images.  An advantage of GANs is that image generation is
%instantaneous
fast once the GAN is trained~\citep{xiao2022tackling}, 
although image samples are of lower quality than for the best diffusion models~\citep{ho2020denoising,rombach2022highresolution}. 
When learning a GAN model, the main challenge arises due to the presence of the generator, which must be trained 
adversarially alongside the discriminator. This requires careful hyperparameter tuning~\citep{brock2019large,karras2020analyzing,liu2020generative},
without which GANs may suffer from training instability and mode collapse~~\citep{arora2017generalization,kodali2017convergence,salimans2016improved}.

Nonetheless, the process of GAN design has given rise to a strong understanding of discriminator functions, and a wide variety
of different divergence measures have been applied.  These fall broadly into two categories: the integral probability metrics
 (among which, the Wasserstein distance \citep{arjovsky2017wasserstein,gulrajani2017improved,genevay18a} and the Maximum Mean Discrepancy \citep{li2017mmd,binkowski2021demystifying,arbel2018gradient})
 and  the f-divergences \citep{goodfellow2014generative,nowozin2016fgan,mescheder2018training,brock2019large}.
 While it would appear that f-divergences ought to suffer from the same shortcomings as diffusions when the target 
 distribution is supported on a submanifold \cite{arjovsky2017wasserstein}, 
 the divergences used in GANs are in practice variational lower bounds on their corresponding f-divergences \citep{nowozin2016fgan}, and in fact behave closer to IPMs
 in that they do not require overlapping support of the target and generator samples,
 and can metrize weak convergence \citep[Proposition 14]{arbel2021generalized} and \citep{zhang18discriminator}
 (there remain important differences, however: notably, f-divergences and their variational lower bounds need not be symmetric in their arguments).

A natural question then arises: is it  possible to define a Wasserstein gradient flow \citep{ambrosioGradientFlowsMetric2008,santambrogio2015optimal}  using a GAN discriminator as a divergence measure? In this setting,
the divergence (discriminator) provides a gradient field directly onto a set of particles (rather than to a generator), transporting them to the target distribution.
Contributions in this direction include the MMD flow \cite{arbel2019maximum,hertrich2023generative}, which defines a Wasserstein Gradient Flow on the Maximum Mean Discrepancy
\citep{gretton12a}; and the KALE (KL approximate lower-bound estimator) flow \cite{glaser2021kale}, which defines a Wasserstein gradient flow on a KL lower bound of the kind used as a GAN discriminator based on an f-divergence \citep{nowozin2016fgan}.
We describe the MMD and its corresponding Wasserstein gradient flow in Section \ref{sec:background}. 
These approaches employ fixed function classes (namely, reproducing kernel Hilbert spaces) for the divergence, and are thus not suited to high dimensional settings such as images. 
Moreover, we show in this work that even for simple examples in low dimensions, an adaptive discriminator ensures faster convergence of a source distribution
to the target, see Section \ref{sec:theoretical_justification}. 

A number of more recent approaches employ trained neural net features in divergences for a subsequent gradient flow \citep[e.g.][]{fan2022variational,franceschi2023unifying}.
% \arxiv{flag} \citep[e.g.][]{altekruger2023neural,franceschi2023unifying}.
Broadly speaking, these works used adversarial means to train a {\em series} of discriminator functions, which are then applied in sequence to a population of particles.
While  more successful on images than kernel divergences,
the approaches retain two shortcomings: they still require adversarial training (on their own prior output), with all the challenges that
this entails; and their empirical performance falls short in comparison with modern diffusions and GANs (see related work in Section \ref{sec:related_work} for details). 

In the present work, we propose a novel Wasserstein Gradient flow on a noise-adaptive MMD divergence measure, leveraging insights from both GANs and diffusion models. 
To {\em train the discriminator}, we start with clean data, and use a forward diffusion process from \citep{ho2020denoising} to produce noisy versions of the data with  given levels of noise (data with  high levels of noise are analogous to the output of a poorly trained generator, whereas low noise  is analogous to a well trained generator). The added noise is always Gaussian.
% \valentin{Maybe this sentence could be skipped. Can keep for now}.
For a given level of noise, we train a noise conditional MMD discriminator to distinguish between the clean and the noisy data, using a single network across all noise levels. This allows us to have better control over the discriminator training procedure than would be achievable with a GAN generator at different levels of refinement, where this control is implicit and hard to characterize. 

To {\em draw new samples},  we propose a novel noise-adaptive version of MMD gradient flow~\citep{arbel2019maximum}. We start from samples drawn from a Gaussian distribution, and move them in the direction of the target distribution by following MMD Gradient flow~\citep{arbel2019maximum}, adapting our MMD discriminator  to the corresponding  level of noise.
% as we go,  from  highest (source) to  lowest (target).
Details may be found in Section \ref{sec:method}.
% Starting from a sample from Gaussian distribution, we follow the MMD gradient flow~\citep{arbel2019maximum} with a controlled level of noise, starting from the highest and going to the lowest noise level.
This allows us to have a fine grained control over the sampling process. 
As a final challenge, MMD gradient flows have previously required large populations of interacting particles for the generation of novel samples, which is expensive (quadratic in the  number of particles) and impractical. In Section \ref{sec:fastSampling}, we propose a scalable approximate sampling procedure for a case of a linear base kernel, which allows {\em single} samples to be generated with a very little loss in quality, at cost idependent of the number of particles used in training.
% \arxivag{
The MMD is an instance of an integral probability metric, however many GANs have been designed using discriminators derived from f-divergences. Section \ref{sec:f_divergences} demonstrates how our approach can equally be applied to such divergences, using a lower bound on the KL divergence as an illustration. Section \ref{sec:related_work} contains a review of alternative approaches to using GAN discriminators for sample generation. 
Finally,
% }
in  Section \ref{sec:experiments}, we show that our method, Diffusion-MMD-gradient flow ($\DMMD$),  yields competitive performance in generative modeling on simple 2-D datasets as well as in unconditional image generation on CIFAR10~\citep{krizhevsky2009learning}, MNIST, CELEB-A (64 x64) and LSUN Church (64 x 64).

\section{Background}
\label{sec:background}
In this section, we define the MMD as a GAN discriminator, then describe Wasserstein gradient flow as it applies for this divergence measure. 
% \valentin{background section is too long}
\paragraph{MMD GAN.} Let $\mathcal{X} \subset \rset^{D}$ and $\mathcal{P}(\mathcal{X})$ be the set of probability distributions defined on $\mathcal{X}$. Let $P \in \mathcal{P}(\mathcal{X})$ be the \textit{target} or data distribution and $Q_{\psi} \in \mathcal{P}(\mathcal{X})$ be a distribution associated with a \emph{generator} parameterized by $\psi \in \rset^{L}$. Let $\mathcal{H}$ be Reproducing Kernel Hilbert Space (RKHS), see \citep{learning_with_kernels} for details, for some kernel $k : \mathcal{X} \times \mathcal{X} \rightarrow \rset$.
% For a kernel $k : \mathcal{X} \times \mathcal{X} \rightarrow \rset$, let $\mathcal{H}$ be Reproducing Kernel Hilbert Space (RKHS) (see \citep{learning_with_kernels}) \valentin{missing words here}.
The Maximum Mean Discrepancy (MMD)~\citep{gretton12a} between $Q_{\psi}$ and $P$ is defined as $\MMD(Q_{\psi}, P) = \sup_{\|f\|_{\mathcal{H}}\le 1} \{\mathbb{E}_{Q_{\psi}} [f(X)] - \mathbb{E}_{P} [f(X)] \}$. We refer to the function $f_{Q_{\psi},P}$ that attains the supremum as the {\em witness function},
\begin{equation}
    \textstyle
    \label{eq:witness_function_new}
    f_{Q_{\psi}, P}(z) \propto \int k(x, z) \rmd Q_{\psi}(x) - \int k(y, z) \rmd P(y),
\end{equation}
which will be essential in defining our gradient flow.
% \begin{align}
% MMD
%     % \label{eq:mmd}
%     % \textstyle
%     \MMD(Q_{\psi}, P) = \sup_{g \in \mathcal{H}} \{\mathbb{E}_{Q_{\psi}} [g(X)] - \mathbb{E}_{P} [g(X)] \}
% \end{align}
Given $X^N = \{x_{i}\}_{i=1}^{N} \sim Q_{\psi}^{\otimes N}$ and $Y^M = \{y_{i}\}_{i=1}^{M} \sim P^{\otimes M}$, the empirical witness function is known in closed form,
\begin{equation}
    \textstyle 
    \label{eq:witness_function_empirical_theory}
    \hat f_{Q_{\psi},P} (x) \propto
    \tfrac{1}{N} \sum_{i =1}^{N} k(x_i,x) - 
    \tfrac{1}{M} \sum_{j =1}^{M} k(y_j,x), 
\end{equation}
and an unbiased estimate of $\MMD^2$~\citep{gretton12a} is likewise straightforward.
% \begin{align}
%     &\textstyle
%     \label{eq:unbiased_mmd}
%     \MMD^2_u[X^N, Y^M] = \tfrac{1}{N (N-1)} \sum_{i \neq j}^{N} k(x_i, x_j) + \\
%     &\textstyle
%     \tfrac{1}{M (M-1)} \sum_{i \neq j}^{M} k(y_i, y_j) -
%     \tfrac{2}{N M} \sum_{i=1}^{N} \sum_{j=1}^{M} k(x_i, y_j) . \nonumber
% \end{align}
In the MMD GAN~\citep{binkowski2021demystifying,li2017mmd}, the kernel %in the objective~\eqref{eq:unbiased_mmd} 
is written
\begin{equation}
    \label{eq:mmd_gan_kernel}
    \textstyle
    k(x, y) = \kbase(\phi(x; \theta), \phi(y; \theta)),
\end{equation}
where $\kbase$ is a base kernel and $\phi(\cdot; \theta): \mathcal{X} \rightarrow \rset^{K}$ are neural networks \emph{discriminator} features with parameters $\theta \in \rset^{H}$. We use the modified notation  $\MMD^2_u[X^N, Y^M; \theta]$ %for equation~\eqref{eq:unbiased_mmd}
to highlight the functional dependence on the discriminator parameters. The $\MMD$ is an Integral Probability Metric (IPM) \citep{muller97}, and thus well defined on distributions with disjoint support: this argument was made in  favor of IPMs by \citet{arjovsky2017wasserstein}. 
Note further that  the Wasserstein GAN discriminators of \citet{arjovsky2017wasserstein,gulrajani2017improved} can be understood in the MMD framework, when the base kernel 
is linear. Indeed, it was observed by \citet{genevay18a} that requiring closer approximation to a true Wasserstein distance resulted in decreased performance in GAN image generation, likely due to the the exponential dependence of sample complexity on dimension for the exact computation of the Wasserstein distance; this motivates an interpretation of these discriminators simply as  IPMs using a class of linear functions of learned features.
We further note that the variational lower bounds used in approximating f-divergences for GANs share the property of being well defined on distribtions with disjoint support \cite{nowozin2016fgan,arbel2021generalized},
although they need not be symmetric in their arguments.
Finally, while  $Q_{\psi}$ and $\theta$ are trained adversarially in GANs, our setting will only require us to learn the discriminator parameter $\theta$.

\paragraph{Wasserstein gradient flows.} An alternative to using a GAN generator is to instead move a sample of particles along the Wasserstein Gradient flow associated with the discriminator~\citep{ambrosioGradientFlowsMetric2008}.
% from a discriminator or any measure of divergence between two distributions \valentin{not super clear starting from "from a discriminator". I would say, "associated with a discriminator" and stop here}.
Let $\mathcal{P}_{2}(\mathcal{X})$ be a set of probability distributions on $\mathcal{X}$ with a finite second moment equipped with the 2-Wasserstein distance.
%AG: removed this the definition of W2 to save space.
%defined for all $\mu, \nu \in \calP_2(\calX)$ as $ W^2_{2}(\nu, \mu) = \inf_{\pi \in \Pi(\nu, \mu)} \int \| x - y \|^2 d \pi(x, y)$,
% \begin{equation}
% \textstyle
%     W^2_{2}(\nu, \mu) = \inf_{\pi \in \Pi(\nu, \mu)} \int \| x - y \|^2 d \pi(x, y) ,
% \end{equation}
% where $\Pi(\nu, \mu)$ contains all possible distributions (couplings) $\pi$ defined on $\mathcal{X} \times \mathcal{X}$ such that $\pi_0 = \mu$ and $\pi_1 = \nu$. 
Let $\mathcal{F}(\nu): \mathcal{P}_{2}(\mathcal{X}) \rightarrow \rset$ be a functional defined over $\mathcal{P}_{2}(\mathcal{X})$ with a property that $\arg \inf_{\nu} \mathcal{F}(\nu) = P$. We consider the problem of transporting mass from an initial distribution $\nu_0 = Q$ to a target distribution $\mu=P$, by finding a continuous path $(\nu_t)_{t \geq 0}$ starting from $\nu_0$ that converges to $\mu$. This problem is studied in Optimal Transport theory~\citep{villani2008optimal,santambrogio2015optimal}. This path can be discretized as a sequence of random variables $(X_n)_{n \in \nset}$ such that $X_n \sim \nu_n$, described by
\begin{equation}
    \label{eq:discrete_dynamics_new}
    X_{n+1} = X_{n} - \gamma \nabla \mathcal{F}'(\nu_n)(X_n), \quad  X_0 \sim Q,
\end{equation}
where $\eta > 0$ and $\mathcal{F}'(\nu_n)(X_n)$ is the first variation of $\mathcal{F}$ associated with the Wasserstein gradient, see~\citep{ambrosioGradientFlowsMetric2008,arbel2019maximum}
% and see Appendix~\ref{sec:optimal_transport}
for precise definitions. As $n \rightarrow \infty$ and $\gamma \rightarrow 0$, depending on the conditions on $\mathcal{F}$, the process ~\eqref{eq:discrete_dynamics_new} will convergence to the gradient flow as a continuous time limit \citep{ambrosioGradientFlowsMetric2008}. %This process corresponds to the (discretised) Wasserstein gradient flow.

\paragraph{MMD gradient flow.} For a choice $\mathcal{F}(\nu) = \MMD^2[\nu, P]$ and a fixed kernel, conditions for convergence of the process in~\eqref{eq:discrete_dynamics_new} to $P$  are given by \citet{arbel2019maximum}. Moreover, the first variation of $\mathcal{F}'(\nu) = f_{\nu, P} \in \mathcal{H}$ is the witness function defined earlier.\footnote{In the case of variational lower bounds on f-divergences, the witness function is still well defined, and the first variation takes the same form in respect of this witness function: see \cite{glaser2021kale} for the case of the KL divergence.}
% called \emph{witness function}, is defined as
% This quantity is called the \emph{witness function}
Using \eqref{eq:witness_function_new}-\eqref{eq:discrete_dynamics_new}, the discretized MMD gradient flow for any $n \in \nset$ is given by
% equation~\eqref{eq:discrete_dynamics_new} and equation~\eqref{eq:witness_function_new}, \valentin{No need to say "equation eqref" just say eqref}
% the discretized MMD gradient flow is given for any $n \in \nset$ by 
\begin{equation}
    \label{eq:mmd_grad_flow_new}
    X_{n+1} = X_{n} - \gamma \nabla f_{\nu_{n}, P}(X_n), \qquad X_0 \sim Q .
\end{equation}
This  provides an algorithm to (approximately) sample from the target distribution $P$.
We remark that \citet{arbel2019maximum,hertrich2023generative} used a kernel with fixed hyperparameters.
In the next section, we will argue that even for RBF kernels (where only the bandwidth is chosen),
faster convergence will be attained using kernels that adapt during the gradient flow. Details of kernel choice
for alternative approaches are given in related work (Section \ref{sec:related_work}).

\section{A motivation for adaptive kernels}
\label{sec:theoretical_justification}

In this section, we demonstrate the benefit of using an \emph{adaptive} kernel when performing MMD gradient flow. We show that even in the simple setting of Gaussian sources and targets, an adaptive kernel improves the convergence of the flow.
% We also study behavior of an optimal kernel.
Consider the following normalized Gaussian kernel,
\begin{equation}
    k_\alpha(x,y) = \alpha^{-d} \exp[-\| x - y \|^2 / (2 \alpha^2)] . 
\end{equation}
For any $\mu \in \rset^d$ and $\sigma > 0$ we denote by $\pi_{\mu, \sigma}$ the Gaussian distribution with mean $\mu$ and covariance matrix $\sigma^2 \Id$. We denote $\MMD_\alpha$ the $\MMD$ associated with $k_\alpha$.

\begin{proposition}
\label{prop:optimal_kernel_gradient}
For any $\mu_0 \in \rset^d$ and $\sigma >0$, let $\alpha^\star$ be given by 
\begin{equation}
\label{eq:maximization_gradient_norm}
    \textstyle \alpha^\star = \argmax_{\alpha \geq 0} \| \nabla_{\mu_0} \MMD^2_\alpha(\pi_{0, \sigma}, \pi_{\mu_0, \sigma})\| .
\end{equation}
Then, we have that 
\begin{equation}
\label{eq:optimal_alpha}
    \alpha^\star = \mathrm{ReLU}(\| \mu_0 \|^2/(d+2) - 2 \sigma^2 )^{1/2} . 
\end{equation}
\end{proposition}

The result is proved in  \Cref{section:proof_theoretical_justification}. The quantity $\| \nabla_{\mu_0} \MMD^2_\alpha(\pi_{0, \sigma}, \pi_{\mu_0, \sigma})\|$ represents how much the mean of the Gaussian $\pi_{\mu_0, \sigma}$ is displaced by a flow 
%AG: below is true but sounds awkward to write, and is clear
%on the mean of the Gaussian
w.r.t. $\MMD^2_\alpha$. Intuitively, we want $\| \nabla_{\mu_0} \MMD^2_\alpha(\pi_{0, \sigma}, \pi_{\mu_0, \sigma})\|$ to be as large as possible as this represents the \emph{maximum displacement possible}. 

We show that $\alpha^\star$ maximizing this displacement is given by \eqref{eq:optimal_alpha}. It is notable that assuming that when $\sigma > 0$ is fixed, this quantity depends on $\| \mu_0 \|$, i.e. the distance between the two distributions. This observation justifies our approach of following an \emph{adaptive} MMD flow at inference time. 
We further highlight the phase transition behaviour of \Cref{prop:optimal_kernel_gradient}:
once the Gaussians are sufficiently close, the optimal kernel width is zero (note that this phase transition would not be
observed in the simpler Dirac GAN example of \citet{mescheder2018training}, where the source and target distributions are Dirac masses with no variance). 
This phase transition suggests that the flow associated with $\MMD$ benefits \emph{less} from adaptivity as the supports of the distributions overlap. We exploit this observation by introducing an optional denoising stage to our procedure; see the end of \Cref{sec:method}.
% We exploit this observation in the final denoising stage of our procedure; see the end of \Cref{sec:method}.

%AG: we 
%This phase transition behavior is also observed in our experiments, see Figure \ref{fig:grad_flow_gaussian}.

In practice, it is not desirable to approximate the distributions of interest by Gaussians, and richer neural network kernel features $\phi(x; \theta)$ are used (see \Cref{sec:experiments}). Approaches to optimize the MMD parameters for GAN training are described by \citet{arbel2018gradient}, which serve as proxies
for convergence speed: notably, it is not sufficient simply to maximize the MMD, since the witness function should remain Lipschitz to ensure convergence \citep[Proposition 2]{arbel2018gradient}. This is achieved in practice by controlling the gradient of the witness function; we take a similar approach in \Cref{sec:method}.

\section{Diffusion Maximum Mean Discrepancy Gradient Flow}
\label{sec:method}

In this section, we present \emph{Diffusion Maximum Mean Discrepancy gradient flow} ($\DMMD$), a new generative model with a training procedure of $\MMD$ discriminator which does not rely on adversarial training, and leverages ideas from diffusion models. The sampling part of $\DMMD$ consists in following a noise adaptive variant of $\MMD$ gradient flow.

% trains a noise-conditional MMD discriminator without adversarial training and

% a method of training a discriminator without adversarial training which is then used to generate new samples using a noise adaptive variant of MMD gradient flow. \valentin{You say here that DMMD is a training method. I would say that DMMD is a new generative model with a training procedure of a discriminator network which does not rely on adversarial training and leverage ideas from diffusion models. The sampling part of DMMD follows a noise adaptive variant of  MMD flow.}

\subsection{Adversarial-free training of noise conditional discriminators} 
\label{sec:mmd_training}

In order to train a discriminator without adversarial training, we propose to use insights from GANs training. In a GAN setting, at the beginning of the training, the generator is randomly initialized and therefore produces samples close to random noise.
% \valentin{again the "poor" "weak" "strong" terminology is subjective at best}.
This would produce a coarse discriminator since it is trained to distinguish clean data from random noise.
% The corresponding discriminator is going to be coarse since it is only trained to distinguish such random noise from the data
% Hence, only a weak discriminator is needed in order to distinguish generated from true samples.
As the training progresses and the generator improves so does the discriminative power of the discriminator. This behavior of the discriminator is central in the training of GANs~\citep{goodfellow2014generative}. We propose a way to replicate this gradually improving behavior without adversarial training and instead relying on principles from diffusion models~\citep{ho2020denoising}.

% When we train a GAN, initially our generator produces poor samples which in turn make the discriminator weak. When generator improves, the discriminator also improves. \valentin{I understand the idea but I think this should be clarified. Maybe something along the lines: In a GAN setting, at the beginning of the training the generator is uninitialized and therefore produce poor samples. Hence, only a weak discriminator is needed in order to distinguish generated from true samples. As the training progresses and the generator improves so does the discriminative power of the discriminator. This behavior of the discriminator is central in the training of GANs (CITE)} We propose a way to replicate this monotonically improving behaviour without adversarial training and instead relying on principles from diffusion models~\citep{ho2020denoising}.

% \valentin{remark: to emphasize a word don't use textbf but use emph}
The forward process in diffusion models allows us to generate a probability path $P_{t}, t \in [0, 1]$, such that $P_0 = P$, where $P$ is our target distribution and $P_1 = \mathrm{N}(0, \Id)$ is a Gaussian noise. Given samples $x_0 \sim P_0 = P$, the samples $x_t | x_0$ are given by 
\begin{equation}
    \label{eq:forward_diffusion}
    x_t = \alpha_t x_0 + \beta_t \epsilon, \quad \epsilon \in \mathrm{N}(0, \Id),
\end{equation}
with $\alpha_0=\beta_1=1$ and $\alpha_1=\beta_0=0$\footnote{Different schedules $(\alpha_t, \beta_t)$ are available in the literature. We focus on Variance Preserving SDE ones \cite{song2020score} here}. From \eqref{eq:forward_diffusion}, we observe that for low noise level $t$, the samples $x_t$ are very close to the original data $x_0$, whereas for the large values of $x_t$ they are close to a unit Gaussian random variable. Using the GANs terminology, $x_t$ could be thought as the output of a generator such that for high/low noise level $t$, it would correspond to \emph{undertrained} / \emph{well-trained} generator.

% such that for $t \approx 1$, it would correspond to an \emph{undertrained} generator, whereas for $t \approx 0$, it would correspond to a \emph{well-trained} generator.
% we could interpret the process from equation~\eqref{eq:forward_diffusion} as a generator, such that $t \sim 1$ corresponds to \emph{undertrained} generator, whereas $t \sim 0$ corresponds to \emph{well-trained} generator.  \valentin{Same as before you use "strong" here. Also is it really the process that can be seen as a generator? I would say something like, "Xt can be thought as the output of a generator"}

Using this insight, for each noise level $t \in [0, 1]$, we define a discriminator $\MMD^2(P_t, P; t, \theta)$ using the kernel of type \eqref{eq:mmd_gan_kernel} with noise-conditional discriminator features $\phi(x; t; \theta)$ parameterized by a Neural Network with learned parameters $\theta$. We consider the following noise-conditional loss function
\begin{equation}
    \label{eq:mmd_loss}
    \calL(\theta, t) =  - \MMD^2(P_t, P; t, \theta)
\end{equation}
% $\calL(\theta, t) =  - \MMD^2(P_t, P; t, \theta)$
% \begin{equation}
%     \label{eq:noise_conditional_discriminator_objective}
%     \textstyle
%     \calL(\theta, t) =  - \MMD^2(P_t, P; t, \theta),
% \end{equation}
where the minus sign comes from the fact that our aim is to maximize the squared MMD. In addition, we regularize this loss
% training of the discriminator in the equation~\eqref{eq:noise_conditional_discriminator_objective}
with $\ell_2$-penalty~\citep{binkowski2021demystifying} denoted $\calL_{\ell_2}(\theta, t)$ as well as with the gradient penalty~\citep{binkowski2021demystifying, gulrajani2017improved} denoted $\calL_{\nabla}(\theta, t)$, see Appendix~\ref{app:noise_dependent_mmd} for the precise definition of these two losses. The total noise-conditional loss is then given as
\begin{equation}
    \label{eq:total_noise_conditional_loss}
    \calL_{\text{tot}}(\theta, t) = \calL(\theta, t) + \lambda_{\ell_2} \calL_{\ell_2}(\theta, t) + \lambda_{\nabla} \calL_{\nabla}(\theta, t),
\end{equation}
for a suitable choice of hyperparameters $\lambda_{\ell_2} \geq 0, \lambda_{\nabla} \geq 0$. Finally, the total loss is given as
\begin{equation}
    \label{eq:total_loss}
    \calL_{\text{tot}}(\theta) =
    \mathbb{E}_{t \sim U[0,1]} \left[ \calL_{\text{tot}}(\theta, t) \right],
\end{equation}
where $U[0,1]$ is a uniform distribution on $[0,1]$. In practice, we use sampled-based unbiased estimator of MMD, see Appendix~\ref{app:noise_dependent_mmd}. The procedure is described in Algorithm~\ref{alg:mmd_training}.

\begin{algorithm}[tb]
   \caption{Train noise-conditional $\MMD$ discriminator}
   \label{alg:mmd_training}
\begin{algorithmic}
   \STATE {\bfseries Input:} Dataset $\mathcal{D} = \{ x_i \}_{i=1}^{N}$
   \STATE Discriminator features $\phi(x, t; \theta)$ with parameters $\theta \in \rset^K$
%   \STATE $\lambda_{\nabla} \geq 0$ is the Gradient Penalty strength
%   \STATE $\lambda_{\ell_2} \geq 0$ is the $\ell_2$ Penalty strength
   \STATE $\lambda_{\nabla} \geq 0,\lambda_{\ell_2} \geq 0$ - gradient and $\ell_2$ penalty coefficients
   \STATE $\gamma > 0$ is the learning rate
   \STATE $\Niter$ is the number of iterations, $B$ is a batch size
   \STATE $\Nnoise$ is the number of noise levels per batch
   \FOR{$i=1$ {\bfseries to} $\Niter$}
     \STATE Sample a batch $B$ of clean data $X_0 \sim P(X_0)$
     \FOR{$n=1$ {\bfseries to} $\Nnoise$}
        \STATE Sample noise level $t_n \sim U[0, 1]$
        \STATE Sample $X_{t_n} \sim p(X_{t_n} | X_0, t_n)$ using eqn.~\eqref{eq:forward_diffusion}
    %  \STATE Sample noise levels $t_n \sim U[0, 1],n=1,\ldots,N_{\text{noise}}$
    %  \STATE Sample $X_{t_n} \sim p(X_{t_n} | X_0, t_n)$ using eqn.~\eqref{eq:forward_diffusion}
        \STATE Let $\phi^{X_0}_{t_n} = \phi(X_0, t_n; \theta)$, $\phi^{X_{t_n}}_{t_n} = \phi(X_{t_n}, t_n; \theta)$
        % \STATE Compute $\MMD$ loss \eqref{eq:mmd_loss} using \eqref{eq:mmd_optimized_app}
        \STATE For linear kernel, compute $\MMD$ loss \eqref{eq:mmd_loss} using \eqref{eq:mmd_optimized_app}
        \STATE Compute the loss $\calL_{\text{tot}}(\theta, t_{n})$ using eqn.~\eqref{eq:total_noise_conditional_loss}
    \ENDFOR
    \STATE Compute total loss $\calL_{\text{tot}}(\theta) = \frac{1}{\Nnoise} \sum_{n=1}^{\Nnoise} \calL_{\text{tot}}(\theta, t_{n})$
    %  \STATE $\calLtot(t_n, \theta) = \calL(t_n, \theta) + \lambda_{\ell_2} \calL_{\ell_2}(t_n, \theta) + \lambda_{\nabla} \calL_{\nabla}(t_n, \theta)$
     \STATE Update discriminator features
     \STATE $\theta \leftarrow \mathrm{ADAM}(\theta, \calL_{\text{tot}}(\theta),\gamma)$
   \ENDFOR
\end{algorithmic}
\end{algorithm}
% \valentin{Instead of Gradient descent for $\theta$ in ALg 1 maybe write $\theta \leftarrow \text{ADAM}(...)$}

\subsection{Adaptive gradient flow sampling}
\label{sec:sampling}

In order to produce samples from $P$, we use the adaptive $\MMD$ gradient flow with noise conditional discriminators $\MMD^2[P_t, P; t; \theta^\star]$, where $\theta^\star$ are the discriminator parameters obtained using Algorithm~\ref{alg:mmd_training}. Let $t_{i} = \tmin + i \Delta t,i=0,\ldots,T$ be the noise discretisation, where $\Delta t = (\tmax - \tmin)/T$
% \valentin{I have corrected but avoid using frac in the text}
such that $t_{0} = \tmin, t_{T}=\tmax$ for some $\tmin=\epsilon$ and $\tmax=1-\epsilon$, where $\epsilon \mll 1$.
% \valentin{dont use sim for approximate. At this stage the reader does not know that we are not putting 0 for numerical stability, say tmin = eps and tmax = 1 with eps << 1}.
We sample $\Npart$ initial particles $\{Z^i  | Z^i \sim \mathrm{N}(0, \Id)\}_{i=1}^{\Npart}$. For each $t$, we follow MMD gradient flow~\eqref{eq:mmd_grad_flow_new} for $\Nfixed$ steps with learning rate $\eta > 0$
\begin{equation}
    \textstyle
    \label{eq:adaptive_mmd_grad_flow}
    Z^{i, n+1}_t = Z^{i, n}_t - \eta \nabla f_{\nu_{\Npart, n}^{t}, P}(Z^{i, n}_t, t; \theta^\star).
\end{equation}
Here $\nu_{\Npart, n}^{t} = 1/\Npart \sum_{i=1}^{\Npart} \updelta_{Z^t_{i, n}}$ is the empirical distribution of particles $\{Z^{i, n}_t\}_{i=1}^{\Npart}$ at the noise level $t$ and the iteration $n$, $\updelta$ is a Dirac mass measure.
The function $f_{\nu_{\Npart, n}^{t}, P}(z, t; \theta^\star)$ is adapted from equation~\eqref{eq:witness_function_new} where $\nu$ is replaced by this empirical distribution.
% and is defined as
% \begin{align}
%     \textstyle
%     f_{\nu_{n}^{t}, P}(z, t; \theta^\star) = \frac{1}{K}\sum_{k=1}^{K} \kbase(\phi(Z^t_{n,k}, t; \theta^\star), \phi(z, t; \theta^\star)) - \\
%     \label{eq:witness_function_discrete} \textstyle
%     \int \kbase(\phi(x_0, t; \theta^\star), \phi(z, t; \theta^\star)) \rmd P(x_0).
% \end{align}
After following the gradient flow~\eqref{eq:adaptive_mmd_grad_flow} for $\Nfixed$ steps, we initialize a new gradient flow with initial particles $Z_{t-\Delta t}^{i, 0} = Z^{i, \Nfixed}_t$ for each $i=1,\ldots,\Npart$, with the decreased level of noise $t - \Delta t$. The recurrence is initialized with $Z_{\tmax}^{i, 0} = Z^{i}$ where $\{Z^i\}_{i=1}^{\Npart}$ are the initial particles. This procedure corresponds to running $T+1$ consecutive MMD gradient flows for $\Nfixed$ iterations each, gradually decreasing the noise level $t$ from $\tmax$ to $\tmin$.
% \valentin{tmin and tmax to be consistent}.
The resulting particles $\{Z_{\tmin}^{i, \Nfixed}\}_{i=1}^{\Npart}$
% \valentin{same use tmin}
are then used as samples from $P$. The procedure is described in Algorithm~\ref{alg:sampling_from_gradient_flow}.
% \subsection{Remark}
% \paragraph{Remark}

In practice, we sample (once) a large batch $N_c$ of $\{X^j_0\}_{j=1}^{N_c} \sim P^{\otimes N_c}$ from the data distribution and denote by $\hat{P}_{\Nclean}(X_0)$ the corresponding empirical distribution. Then we use the empirical witness function \eqref{eq:witness_function_empirical_theory} given by
% we use empirical witness function \eqref{eq:witness_function_empirical_theory} $f_{\nu_{\Npart, n}^{t}, \hat{P}_{\Nclean}}$ defined on samples $\{X^j_0\}_{j=1}^{N_c} \sim P^{\otimes N_c}$ (sampled once) from the data distribution. This witness function is given as
% In practice, we use empirical 
% To compute $f_{\nu_{\Npart, n}^{t}, P}(z, t; \theta^\star)$ from \eqref{eq:witness_function_new},
% % using eqn.~\eqref{eq:witness_function_new} \valentin{general remark do not use equation eqref or eqn eqref, just use eqref}
% we need to compute the integral $\int k(y, z) \rmd P(y)$ with respect to the distribution $P$. In practice for a general kernel $k$, we use Monte-Carlo approximation of this integral using samples $\{X^j_0\}_{j=1}^{N_c} \sim P^{\otimes N_c}$ (sampled once). We denote by $\hat{P}_{\Nclean}$ the corresponding empirical distribution. The empirical witness function is
\begin{align}
    & f_{\nu_{\Npart, n}^{t}, \hat{P}_{\Nclean}(X_0)}(z, t; \theta^\star) = \\
    \textstyle
    &\quad \textstyle\frac{1}{\Npart} \sum_{i=1}^{\Npart} \kbase(\phi(Z^{n,i}_t, t; \theta^\star), \phi(z, t; \theta^\star))  \\
    &\qquad \textstyle
    -\frac{1}{\Nclean} \sum_{j=1}^{\Nclean} \kbase(\phi(X^j_0, t; \theta^\star), \phi(z, t; \theta^\star)) .
    \label{eq:witness_function_empirical}
\end{align}

\paragraph{Final denoising.} In diffusion models~\citep{ho2020denoising}, it is common to use a denoising step at the end to improve samples quality. Empirically, we found that doing a few $\MMD$ gradient flow steps at the end of the sampling with a higher learning rate $\eta$ allowed to reduce the amount of noise and improve performance.

% We found empirically that the samples at the end of sampling through $\DMMD$ had some amount of noise. We found that doing a few $\MMD$ gradient flow steps at the end with higher learning rate helped to reduce the amount of noise and improve empirical performance.
% The clean particles
% Note that the clean samples are only sampled once.
% \valentin{Equation too long. May emphasize that the Xj are not recomputed at each step}

% which means that we replace the distribution $P$ by an empirical distribution $\hat{P}_{\Nclean}$
% \begin{equation}
%     \hat{P}_{\Nclean}(y) = \sum_{j=1}^{\Nclean} \delta(y - X^j_0), X^j_0 \sim_{iid} P
% \end{equation}

\begin{algorithm}[tb]
   \caption{Noise-adaptive $\MMD$ gradient flow}
   \label{alg:sampling_from_gradient_flow}
\begin{algorithmic}
   \STATE {\bfseries Inputs:} $T$ is the number of noise levels
   \STATE $\tmax, \tmin$ are maximum and minimum noise levels
   \STATE $\Nfixed$ is the number of gradient flow steps per noise level
   \STATE $\eta > 0$ is the gradient flow learning rate
   \STATE $\Npart$ is the number of noisy particles
   \STATE Batch of clean particles $X_{0} \sim \mathcal{P}_{0}$.
   \STATE {\bfseries Steps:} Sample initial noisy particles $Z \sim \mathrm{N}(0,\Id)$
   \STATE Set $\Delta t = (\tmax - \tmin) / T$
   \FOR{$i=T$ {\bfseries to} $0$}
     \STATE Set the noise level $t = \tmin + i \Delta t$ and $Z^0_t = Z$
     \FOR{$n=0$ {\bfseries to} $\Nfixed - 1$}
        \STATE Compute $f_{\nu_{\Npart, n}^{t}, \hat{P}_{\Nclean}(X_0)}(Z^{n}_t, t; \theta^\star)$ using \eqref{eq:witness_function_empirical}
        \STATE  $Z^{n+1}_t = Z^{n}_{t} - \eta \nabla f_{\nu_{\Npart, n}^{t}, \hat{P}_{\Nclean}(X_0)}(Z^{n}_t, t; \theta^\star)$
     \ENDFOR
     \STATE Set $Z = Z^{N}_t$
   \ENDFOR
   \STATE Output $Z$
\end{algorithmic}
\end{algorithm}

\section{Scalable $\DMMD$ with linear kernel}
% Linear base kernel for a scalable $\MMD$ gradient flow}
\label{sec:fastSampling}
The computation complexity of the $\MMD$ estimate %\eqref{eq:unbiased_mmd}
on two sets of $N$ samples is $O(N^2)$, so as of the witness function \eqref{eq:witness_function_empirical} for $N$ noisy and clean particles.
% $f_{\nu_{\Npart, n}^{t}, \hat{P}_{\Nclean}}(z, t; \theta^\star)$ for $N$ noisy particles will also take $O(N^2)$.
This makes scaling to large $N$ prohibitive. Using linear base kernel (see \eqref{eq:mmd_gan_kernel})
\begin{equation}
    \textstyle
    \label{eq:linear_kernel_mmd}
    \kbase(x, y) = \langle x, y \rangle,
\end{equation}
% \valentin{better to use lange x , y rangle instead of x T y}
allows to reduce the computation complexity of both quantities down to $O(N)$, see Appendix~\ref{app:linear_optimization}. We consider the average noise conditional discriminator features on the \emph{whole} dataset
\begin{equation}
  \textstyle
  \label{eq:avg_discrimiinator_features}
    \bar{\phi}(X_0, t; \theta^\star) = \frac{1}{N} \sum_{i=1}^{N} \phi(X^i_0, t; \theta^\star) .
\end{equation}
Using linear kernel~\eqref{eq:linear_kernel_mmd} allows us to use these average features~\eqref{eq:avg_discrimiinator_features} in the second term of \eqref{eq:witness_function_empirical}. In practice, we can precompute these features for $T$ timesteps and store them in memory in order to use them for sampling purposes. The associated storage cost is $O(TK)$ where $K$ is the dimensionality of these features.

\paragraph{Approximate sampling procedure.}
% \subsection{Approximate sampling procedure}
% \label{sec:approximate_sampling}
$\MMD$ gradient flow~\eqref{eq:adaptive_mmd_grad_flow} requires us to use multiple interacting particles $Z$ to produce samples, where the interaction is captured by the first term in \eqref{eq:witness_function_empirical}. In practice this means that the performance will depend on the number of these particles. In this section, we propose an approximation to $\MMD$ gradient flow with a linear base kernel~\eqref{eq:linear_kernel_mmd} which allows us to sample particles \emph{independently}, therefore removing the need for multiple particles. For a linear kernel, the interaction term in \eqref{eq:witness_function_empirical} for a particle $Z$, equals to
% In this section, we propose i.i.d. sampling approach based on an approximation of MMD gradient flow~\eqref{eq:adaptive_mmd_grad_flow} when using linear base kernel~\eqref{eq:linear_kernel_mmd}. When using the linear kernel, the interaction term for each particle $z$ is given as
\begin{equation}
    \textstyle
    \langle \frac{1}{\Npart} \sum_{i=1}^{\Npart} \phi(Z^{n,i}_t, t; \theta^\star), \phi(Z, t; \theta^\star) \rangle , 
\end{equation}
% \valentin{use scalar product it is more compact. I have added it but be careful with the punctuation of equations}
% \valentin{sometimes you have Z t n i and sometimes Z t i n, check for consistency}
For a large number of particles $\Npart$, the contribution of each particle $Z^t_{n,i}$ on the interaction term with $Z$ will be small. For a sufficiently large $\Npart$, we hypothesize that
\begin{equation}
    \textstyle
    \label{eq:approximation}
    \frac{1}{\Npart} \sum_{i=1}^{\Npart} \phi(Z^{n,i}_t, t; \theta^\star) \approx \frac{1}{N} \sum_{j=1}^{N} \phi(X^j_t, t; \theta^\star),
\end{equation}
% \valentin{use approx and not sim. sim is for sampling }
where $N$ is the size of the dataset and $X^j_t$ are produced by the forward diffusion process~\eqref{eq:forward_diffusion} applied to each $X^j_0$. In Section~\ref{sec:experiments}, we test this approximation in practice.
% where $N$ is the size of the dataset and $X^j_t$ are sampled from $P_t$ given by the forward diffusion process~\eqref{eq:forward_diffusion}.

%AG: I commented this out since I think that what matters is the temperature of the ambient witness.
%In general, the distributions of both quantities will be different, and this approximation might not be valid. However, as discussed in Section~\ref{sec:mmd_training}, larger $t$ would lead to a coarse discriminator, such that it would give similar $\MMD$ for two different noisy distributions. Therefore, we hypothesize that this approximation holds $t$ is large. In Section~\ref{sec:experiments}, we test this approximation in practice.

Using \eqref{eq:approximation}, we consider an approximate witness function
\begin{equation}
    \textstyle
    \label{eq:approx_witness}
    \hat{f}_{P_t, P}(z) = \langle \phi(z, t; \theta^*),  \bar{\phi}(X_t, t; \theta^\star) - \bar{\phi}(X_0, t; \theta^\star) \rangle,
\end{equation}
with $\bar{\phi}(X_t, t; \theta^\star)$ precomputed using \eqref{eq:avg_discrimiinator_features}.
In practice, we sample \emph{single} particle $Z \sim \mathrm{N}(0, \Id)$ and follow noise-adaptive $\MMD$ gradient flow with \eqref{eq:approx_witness}
\begin{equation}
    \textstyle
    \label{eq:approximate_mmd_grad_flow}
    Z^{n+1}_t = Z^n_t - \eta \nabla \hat{f}_{P_t, P}(Z^n_t)
\end{equation}
The corresponding algorithm is described in Appendix~\ref{app:approximate_sampling}.

\section{f-divergences}
\label{sec:f_divergences}

\final{The approach described in Section~\ref{sec:method} can be applied to any divergence which has a well defined Wasserstein Gradient Flow described by a gradient of the associated witness function. 
Such divergences include the variational lower bounds on f-divergences, as described by~\citep{nowozin2016fgan}, which are popular in GAN training, and were indeed the basis of the original GAN discriminator \citep{goodfellow2014generative}.} One  such f-divergence is the KL Approximate Lower bound Estimator \citep[KALE, ][]{glaser2021kale}. Unlike the original KL divergence, which requires a density ratio, the KALE remains well defined for distributions with non-overlapping support. Similarly to $\MMD$, the Wasserstein Gradient of $\KALE$ is given by the gradient of a learned witness function. Therefore, we train noise-conditional $\KALE$ discriminator and use corresponding noise-conditional Wasserstein gradient flow, similarly to $\DMMD$. We call this method \emph{Diffusion $\KALE$ flow} ($\mathrm{D}$-$\KALE$-Flow). The full approach is described in Appendix~\ref{app:d_kale_flow}. \final{We found this approach to lead to reasonable empirical results, but unlike with $\DMMD$, it achieved worse performance than a corresponding GAN, see Appendix~\ref{app:additional_results_f_divergence}.}
%}

\section{Related Work}
\label{sec:related_work}

\paragraph{Adversarial training and $\MMD$-GAN.}
Integral Probability Metrics (IPMs) are good candidates to define discriminators in the context of generative modeling, since they are well defined even in the case of distributions with non-overlapping support \citep{muller97}. 
Moreover, implementations of f-divergence discriminators in GANs rely on variational lower bounds~\citep{nowozin2016fgan}:
as noted earlier, these share useful properties of IPMs in theory and in practice (notably, they remain well defined
for distributions with disjoint support, and may metrize weak convergence for sufficiently rich witness function classes  \citep[Proposition 14]{arbel2021generalized} and \citep{zhang18discriminator}). 
Several works~\citep{arjovsky2017wasserstein,gulrajani2017improved,genevay18a,li2017mmd,binkowski2021demystifying} have exploited  IPMs as discriminators for the training of GANs, where the IPMs are MMDs using (linear or nonlinear) kernels defined on learned neural net features, making them suited to high dimensional settings such as image generation.
%
%Among these IPMs, the squared Maximum Mean Discrepancy ($\MMD$)~\citep{gretton12a} is particularly appealing since it allows training of the discriminator kernel. % the user to choose the kernel~\citep{learning_with_kernels}. 
Interpreting the IPM-based GAN discriminator as a squared  $\MMD$ yields an interesting theoretical insight: \citet{franceschi2022neural} show that training a GAN with an IPM objective implicitly optimizes  $\MMD^2$ in the Neural Tangent Kernel (NTK) limit~\citep{jacot2020neural}. 
 IPM GAN discriminators are trained jointly with the generator in a min-max game. Adversarial training is challenging, and can suffer from instability, mode collapse, and misconvergence~\citep{xiao2022tackling,binkowski2021demystifying,li2017mmd,arora2017generalization,kodali2017convergence,salimans2016improved}. 
 Note that once a GAN has been trained, the samples can be refined via MCMC sampling in the generator latent space \citep[e.g., using kinetic Langevin dynamics; see][]{ansari2021refining,che2021gan,arbel2021generalized}.

\paragraph{Discriminator flows for generative modeling.} Wasserstein Gradient flows~\citep{ambrosioGradientFlowsMetric2008,santambrogio2015optimal} applied to a GAN discriminator are informally called \emph{discriminator flows}, see \citep{franceschi2023unifying}. A number of recent works have focused on replacing a GAN generator by a discriminator flow. \citet{fan2022variational} propose a discretisation of JKO~\citep{jordan1998} scheme to define a Kullback-Leibler (KL) divergence gradient flow. Other approaches have used a discretized interactive particle-based approach instead of JKO, similar to \eqref{eq:discrete_dynamics_new}.  \citet{heng2023deep,franceschi2023unifying} build such a flow based on f-divergences, %\arxiv{
whereas \citet{franceschi2023unifying} focuses on $\MMD$ gradient flow.
%}
In all these works, an explicit generator is replaced by a corresponding discriminator flow. The sampling process during training is as follows: Let $Y_k$ be the samples produced at training iteration $k$ by the gradient flow $\mathcal{F}_{k}$ induced by the discriminator $\mathcal{D}_{k}$ applied to samples $Y_{k-1}$ from the previous iteration. We denote this by $Y_k \leftarrow \mathcal{F}_{k}(\mathcal{D}_{k}, Y_{k-1})$. Then, the discriminator at iteration $k+1$ is trained on samples $Y_k$. 
A challenge of this process is that the training sample for the next discriminator will be determined by the previous discriminators, and thus the generation process is still adversarial: particle transport minimizes the previous discriminator value, and the subsequent discriminator is maximized on these particles. 
Consequently, it is difficult to control or predict the overall sample trajectory from the initial distribution to the target, which might explain the performance shortfall of these methods in image generation settings.
By contrast, we have explicit control over the training particle trajectory via the forward noising diffusion process.

On top of that, these approaches \citep[except for][]{heng2023deep} require to store all intermediate discriminators $\mathcal{D}_{1},\ldots,\mathcal{D}_{N}$ throughout training ($N$ is the total number of training iterations). These discriminators are then used to produce new samples by applying the sequence of gradient flows $\mathcal{F}_{N}(\mathcal{D}_{N}, \cdot) \circ \ldots \circ \mathcal{F}_{1}(\mathcal{D}_{1}, \cdot)$ to $Y_0$ sampled from the initial distribution. This creates a large memory overhead.

An alternative is to use pretrained features obtained elsewhere or a fixed kernel with empirically selected hyperparameters
%\arxiv{
\citep[see][]{hertrich2023generative,hagemann2023posterior,altekruger2023neural}
%}
, however this limits the applicability of the method. To the best of our knowledge, our approach is the first to demonstrate the possibility to train a discriminator without adversarial training, such that this discriminator can then be used to produce samples with a gradient flow. Unlike the alternatives, our approach does not require to store intermediate discriminators.

\paragraph{$\MMD$ for diffusion refinement/regularization.} $\MMD$ has been used to either regularize training of diffusion models~\citep{li2024error} or to finetune them~\citep{aiello2023fast} for fast sampling. The $\MMD$ kernel in these works has the form \eqref{eq:mmd_gan_kernel} with  Inception features~\citep{szegedy2014going}. Our method removes the need to use pretrained features by training th $\MMD$ discriminator.

\paragraph{Diffusion models.} Diffusion models~\citep{sohldickstein2015deep,ho2020denoising,song2020score} represent a powerful new family of generative models due to their strong empirical performance in many domains~\citep{saharia2022photorealistic,le2023voicebox,ho2022imagen,watson2022broadly,poole2022dreamfusion}. Unlike GANs, diffusion models do not require adversarial training. At training time, a denoiser is learned for multiple noise levels. As noted above, our work borrows from the training of diffusion models, as we train a discriminator on multiple noise levels of the forward diffusion process \citep{ho2020denoising}. This gives better control of the training samples for the (noise adapted) discriminator than using an incompletely trained GAN generator.

\section{Experiments}
\label{sec:experiments}

\begin{figure*}[t]
    \centering
    \includegraphics[width=1.0\textwidth]{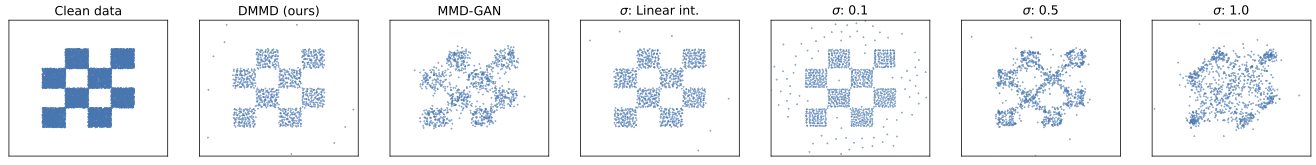}
    \vspace*{-5mm}
    \caption{Samples from $\MMD$ Gradient flow with different parameters for the RBF kernel~\eqref{eq:noise_conditional_rbf_kernel}.}
    \label{fig:2d_samples}
\end{figure*}

\subsection{Understanding $\DMMD$ behavior in 2-D}
\label{sec:toy_experiments}
\begin{figure}[!htb]
    \centering
    \includegraphics[scale=0.2]{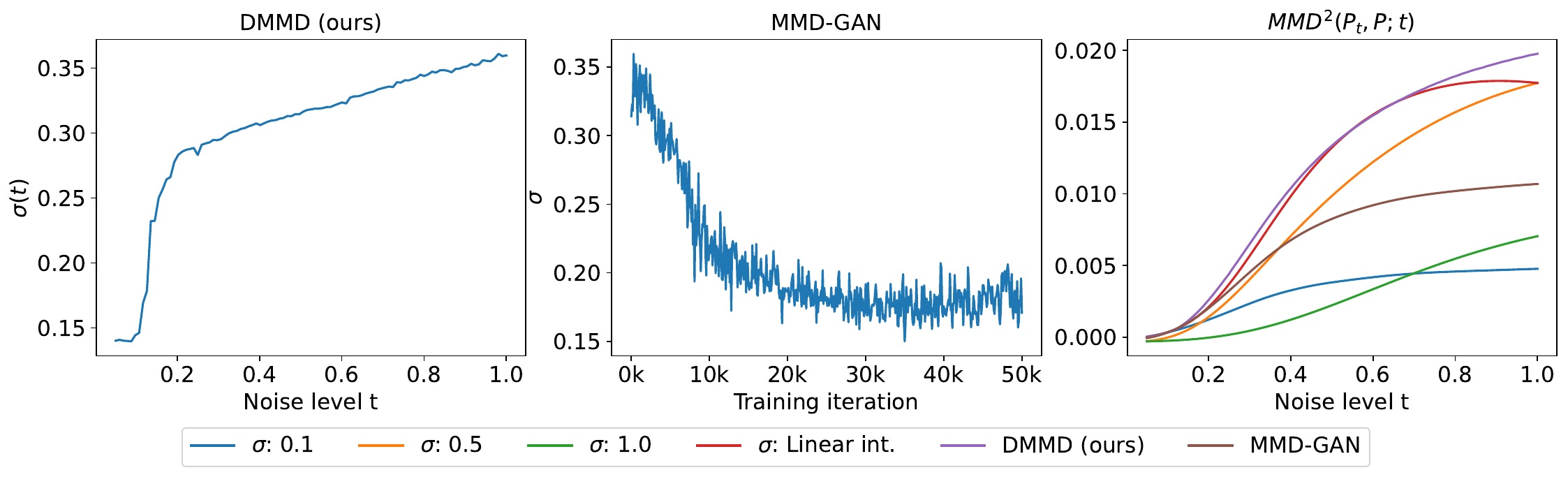}
    \caption{\textbf{Qualitative behaviour of $\MMD$ discriminators}. Left, learned RBF kernel~\eqref{eq:noise_conditional_rbf_kernel} widths $\sigma(t)$ as a function of noise level $t$. Center, parameter $\sigma$ for $\MMD$-GAN as function of training iteration. Right, $\MMD^2(P_t, P; t)$ for different methods.
    % On the left we show the learned RBF kernel~\eqref{eq:noise_conditional_rbf_kernel} widths $\sigma(t)$ as functions of noise levels in our $\DMMD$ method. In the center, we show the learned parameter $\sigma$ when we use $\MMD$-GAN as function of training iterations. In the right, we show the $\MMD^2(P_t, P; t)$ for different methods as a function of noise level.
    }
    \label{fig:2d_qualitative}
\end{figure}
The aim of the experiments in this section is to get an understanding of the behavior of $\DMMD$ described in Section~\ref{sec:method}. We expect $\DMMD$ to mimic GAN discriminator training via noise conditional discriminator learning.
% \textcolor{red}{By maximizing $\MMD^2(P_t, P; t, \theta)$ for all noise levels $t \in [0, 1]$, $\DMMD$ aims to mimic the training of GAN discriminator. For a large level of noise $t$, the distribution $P_t$ is close to Gaussian noise, which means that a corresponding noise conditional $\MMD$ discriminator is similar to an undertrained GAN discriminator. On the opposite side, for a low level of noise $t$, the distribution $P_t$ is close to the data and the corresponding noise conditional $\MMD$ discriminator is close to a well trained GAN discriminator.} \valentin{all of this has already been said, i dont think you need to rewrite the whole explanation, just a pointer will suffice}
To see whether this manifests in practice, we design a toy experiment using Radial Basis Function (RBF) kernel for $\MMD$
% which in turn means that a noise conditional MMD discriminator is similar to an undertrained GAN discriminator. Oppositely, for a low level of noise $t$, $P_t$ is close to $P$ and MMD discriminator is close to a well trained GAN discriminator. To test whether this intuition manifests in practice, we design a toy experiment using Radial Basis Function (RBF) kernel for MMD
% such that for large levels of noise $t$ in the equation~\eqref{eq:forward_diffusion}, the noise-conditional $\MMD^2(P_t, P; t, \theta)$ from equation ~\eqref{eq:noise_conditional_discriminator_objective} mimics a poorly trained GAN discriminator whereas for low level of noise $t$, it mimics a well trained one. To test whether this behavior happens in practice, we use RBF kernel for MMD
\begin{equation}
    \label{eq:noise_conditional_rbf_kernel}
    k_t(x, y) = \exp[-\|x - y\|^2 / (2 \sigma^2(t; \theta))] ,
\end{equation}
where the noise dependent kernel width function $\sigma(\cdot; \theta): [0,1] \rightarrow [0,+\infty)$ is parameterized by $\theta \in \rset^K$. 
%The larger / smaller the $\sigma$, the larger / smaller is the span of $(x, y)$ where the kernel has high values. 
This parameter controls the coarseness of the $\MMD$ discriminator.

% The parameter $\sigma$ is a width of a kernel with a property that if points are far apart more than $\sigma$, the kernel value will exponentially decrease. Therefore, for low level of noise, the noisy data $x_t$ from equation~\eqref{eq:forward_diffusion} will tend to be very close to the clean data, meaning that $||x-y||$ will be small in average. Therefore, in order to well distinguish points from each other, the value of $\sigma$ should also be small. The opposite is true for high level of noise.

We consider a simple checkerboard 2-D dataset, see Figure~\ref{fig:2d_samples},~left. We learn noise-conditional kernel widths $\sigma(t;\theta)$ using a neural network ensuring that $\sigma(t;\theta)>0$.
As baselines, we consider $\MMD$-GAN with a trained generator and a discriminator with one learnable parameter $\sigma$.
% To learn noise-conditional $\MMD$ for $\DMMD$, we use a $4$-layers MLP to encode $\sigma(t;\theta)$ ensuring that $\sigma(t;\theta)>0$.
% we use a $4$-layers MLP $g(t; \theta)$ with ReLU activation to encode $\sigma(t; \theta) = \sigma_{\text{min}} + \text{ReLU}(g(t; \theta))$ with $\sigma_{\text{min}}=0.001$, which ensures $\sigma(t; \theta) > 0$.
% As baselines, we consider $\MMD$-GAN where a generator is a $3$-layers MLP with ReLU activations and the discriminator has only one learnable parameter $\sigma$. 
On top of that, we consider $\MMD$ gradient flow with fixed values of $\sigma$ and a variant called \emph{linear interpolation} with manually chosen noise-dependent $\sigma(t) =  0.1 (1-t) + 0.5 t$. All experimental details are provided in Appendix~\ref{app:2d_toy_datasets}.

% \valentin{details about the architecture of the net can be postponed to the appendix}

We report the learned RBF kernel widths for $\DMMD$ in Figure~\ref{fig:2d_qualitative}, left.
% We use a slightly modified forward diffusion process from DDPM paper~\citep{ho2020denoising} to smooth the noising process out. We train the MMD discriminator~\eqref{eq:total_noise_conditional_loss} for $\Niter = 50000$ iterations with a batch size of $B=256$ and $\Nnoise=256$ levels of noise per batch. We use a $4$-layer MLP $g(t; \theta)$ to encode $\sigma(t; \theta) = \sigma_{\text{min}} + \text{ReLU}(g(t; \theta))$ and select $\sigma_{\text{min}}=0.001$ to ensure that the output is positive. As a baseline, we train an MMD-GAN~\citep{binkowski2021demystifying,li2017mmd}, where a generator is a $3$-layer MLP and the MMD discriminator uses RBF kernel \eqref{eq:noise_conditional_rbf_kernel} where $\sigma$ does not depend on time and is a learnable parameter. On top of that, we use MMD Gradient low~\citep{arbel2019maximum} with a fixed constant value of $\sigma$ as well as with a manually chosen time-dependent $\sigma(t) = (1-t) \sigma_0 + t \sigma_1$ with $\sigma_0 = 0.1$ and $\sigma_1 = 0.5$. We call this variant \emph{linear interpolation}. See Appendix~\ref{app:2d_toy_datasets} for more details. We report the learned RBF kernel widths parameters $\sigma(t)$ in Figure~\ref{fig:2d_qualitative},left.
As expected, as noise level goes from high to low, the kernel width $\sigma(t)$ decreases. In Figure~\ref{fig:2d_qualitative}, center, we show the learned $\MMD$-GAN kernel width parameter $\sigma$ as a function of training iterations. As expected, when the training progresses, this parameter decreases, since the corresponding generator produces samples, close to the target distribution. The behaviors of $\DMMD$ and $\MMD$-GAN are quite similar. Interestingly, the range of values for the kernel widths is also similar. This highlights our point that $\DMMD$ mimics the training of a GAN discriminator. The exact dynamics for $\sigma(t)$ in $\DMMD$ depends on the parameters of the forward diffusion process \eqref{eq:forward_diffusion}. The sharp phase transition is consistent with the phase transition highlighted in Section~\ref{sec:theoretical_justification}. In addition, we report $\MMD^2(P_t, P; t)$ for different methods in Figure~\ref{fig:2d_qualitative}, right. We observe that the behavior of $\DMMD$ is close to \emph{linear interpolation} variant, but is more nuanced for higher noise levels.
% \textcolor{red}{This means that $\DMMD$ is able to distinguish data for \emph{more} noise levels} \valentin{I dont understand}.
Finally, we report the corresponding samples in Figure~\ref{fig:2d_samples}. We see that $\DMMD$ produces visually better samples than the other baselines. For RBF kernel, we noticed the presence of outliers. The amount of outliers generally depends on the kernel, see Appendix of \citep{hertrich2023generative} for more details.

\subsection{Image generation}
We study the performance of $\DMMD$ on unconditional image generation of CIFAR10~\citep{krizhevsky2009learning}. We use the same forward diffusion process as in \citep{ho2020denoising} to produce noisy images. We use a U-Net~\citep{ronneberger2015unet} backbone for discriminator feature network $\phi(x, t; \theta)$, with a slightly different architecture from the one used in \citep{ho2020denoising}, see Appendix~\ref{app:experiments}. For all the image-based experiments, we use linear base kernel~\eqref{eq:linear_kernel_mmd}. We explored using other kernels such as RBF and Rational Quadratic (RQ), but did not find an improvement in performance. We use FID~\citep{heusel2018gans} and Inception Score~\citep{salimans2016improved} for evaluation, see Appendix~\ref{app:experiments}.
% \textcolor{red}{As evaluation metrics, we use FID~\citep{heusel2018gans} and Inception Score~\citep{salimans2016improved} using the same evaluation regime as in \citep{ho2020denoising}. To select hyperparameters and track performance during training, we use FID evaluated on a subset of $1024$ images from a training set of CIFAR10} \valentin{move to appendix. The part on FID and IS will be self explanatory given the table }.
% See Appendix~\ref{app:experiments} for more details.
Unless specified otherwise, we use the number $\Npart = 200$ of particles for Algorithm~\ref{alg:sampling_from_gradient_flow}. We provide ablation over the number of particles in Appenidx~\ref{app:number_of_particles}.
% Algorithm~\ref{alg:sampling_from_gradient_flow}.
% and provide an ablation over this parameter below.
% We used a fixed set of parameters for $\DMMD$ for Algorithm~\ref{alg:sampling_from_gradient_flow} during training to track intermediate performance.
% Since $\DMMD$ has tunable parameters at sampling time (see Algorithm~\ref{alg:sampling_from_gradient_flow}), we use a set of fixed parameters at training (see Appendix~\ref{app:experiments}), but then run an additional sweep over these sampling parameters and select the best ones using the FID score on a subet of CIFAR10.
% we run a sweep over these parameters and select the best using the same FID metric.
% DMMD has tunable parameters for the sampling time (see Algorithm~\ref{alg:sampling_from_gradient_flow}). We run a sweep over this parameters and select the best based on the training FID. See Appendix~\ref{app:experiments} for more details.
% Unless specified otherwise, we use the number of particles $\Npart = 200$ and provide an ablation over this parameter below.
% For the results in this section, we use the number of particles $\Npart = 200$ and we provide an ablation over this parameter below.
The total number of iterations for $\DMMD$ equals to $T \times \Nfixed$, where $T$ is the number of noise levels and $\Nfixed$ is the number of steps per noise level. To be consistent with diffusion models, we call this number \emph{number of function evaluations} (NFE). For $\DMMD$, we report performance with different NFEs.
% The total number of iterations for DMMD equals to $T \times \Nfixed$, see Algorithm~\ref{alg:sampling_from_gradient_flow} for the definition of these parameters. We call this number \emph{number of function evaluations} (NFE) to be consistent with diffusion models.
% We report performance of DMMD with different NFE.
% Since DMMD has tunable parameters at sampling time such as number of noise levels $T$, the minimum $\tmin$ and maximum $\tmax$ noise levels, gradient flow learning rate $\eta$, number of gradient flow steps per noise level $\Nfixed$, number of noisy particles $\Npart$, we sweep over these and select the best parameters based on training time FID. The total number of iterations of DMMD equals to $T \times \Nfixed$. We call this parameter \emph{number of function evaluations} (NFE) to be consistent with diffusion models. For the results in this section, we use $\Npart = 200$ and we provide an ablation over this parameter below.

As baselines we consider our implementation of $\MMD$-GAN~\citep{binkowski2021demystifying} with linear base kernel and DDPM~\citep{ho2020denoising} using the same neural network backbones as for $\DMMD$. We also report results from the original papers. On top of that, we consider baselines based on \emph{discriminator flows}. JKO-Flow~\citep{fan2022variational}, which uses JKO~\citep{jordan1998} scheme for the KL gradient flow. Deep Generative Wasserstein Gradient Flows (DGGF-KL)~\citep{heng2023deep}, which uses particle-based approach (similar to \eqref{eq:discrete_dynamics_new}) for the KL gradient flow.
%\arxiv{
These approaches use adversarial training to train discriminators, see Section~\ref{sec:related_work} for more details.
%}
On top of that, we consider Generative Sliced $\MMD$ Flows with Riesz Kernels (GS-$\MMD$-RK)~\citep{hertrich2023generative} which uses similar particle based approach to DGGF-KL to construct $\MMD$ flow,
%\arxiv{
but uses fixed (kernel) discriminator.
%}
% to construct to construct KL-divergence flow. 
% inally, Generative Sliced $\MMD$ Flows with Riesz Kernels (GS-$\MMD$-RK)~\citep{hertrich2023generative} uses similar particle based approach to DGGF-KL to construct $\MMD$ flow.
% All these methods use adversarial training to train discriminators. Moreover, these approaches require to store \textcolor{red}{intermediate} \valentin{intermedaite what?} throughout the training discriminator, in order to sample from the target distribution. \textcolor{red}{This makes these approaches to be memory-expensive. $\DMMD$ overcomes these issues by design.} \valentin{not very well written, I'll try to think of something to rewrite}
On top of that, we report results using a discriminator flow defined on a trained $\MMD$-GAN discriminator which we call $\MMD$-GAN-Flow.
% Finally, we report results for using $\MMD$ gradient flow from a trained $\MMD$-GAN discriminator which we call $\MMD$-GAN-Flow.
More details on experiments are given in Appendix~\ref{app:experiments}. The results are provided in Table~\ref{table:fids}.
\begin{table}[!htbp]
\small
\centering
\caption{\textbf{Unconditional image generation on CIFAR-10}. For $\MMD$ GAN (orig.), we used mixed-RQ kernel described in \citep{binkowski2021demystifying}. "Orig." stands for the original paper and "impl." stands for our implementation. For JKO-Flow~\citep{fan2022variational}, the NFE is taken from their Figure 12.} 
\label{table:fids}
\begin{tabular}{lccc}
\toprule
Method & FID & Inception Score & NFE \\
\midrule
MMD GAN (orig.) &  39.90 & 6.51 & - \\
MMD GAN (impl.) & 13.62 & 8.93 & - \\
DDPM (orig.) & 3.17 & 9.46 & 1000 \\
DDPM (impl.) & 5.19 & 8.90 & 100 \\
\midrule
 \multicolumn{4}{c}{\textbf{Discriminator flow baselines}} \\
 \midrule
 DGGF-KL & 28.80 & - & 110 \\
 JKO-Flow & 23.10 & 7.48 & $\sim 150$ \\
 \midrule
 \multicolumn{4}{c}{\textbf{MMD flow baselines}} \\ 
 \midrule
 MMD-GAN-Flow & 450 & 1.21 & 100 \\ 
 GS-MMD-RK & 55.00 & - & 86 \\
\midrule
DMMD (ours) & \textbf{8.31} & \textbf{9.09} & 100 \\
DMMD (ours) & \textbf{7.74} & \textbf{9.12} & 250 \\
\end{tabular}
\end{table}
We see that $\DMMD$ achieves better performance than the $\MMD$ GAN. As expected, $\MMD$-GAN-Flow does not work at all. This is because the $\MMD$-GAN discriminator at convergence was trained on samples close to the target distribution. Making a parallel with RBF kernel experiment from Section~\ref{sec:toy_experiments}, this means that the gradient of $\MMD$ will be very small on samples far away from the target distribution. This highlights the benefit of adaptive $\MMD$ discriminators. Moreover, %\arxiv{
we also see that $\DMMD$ performs better than GS-MMD-RK, which uses fixed kernel. This highlights the advantage of learning discriminator features in $\DMMD$.
%}
We see that $\DMMD$ achieves superior performance compared to other discriminator flow baselines. We believe that one of the reasons why these methods perform worse than $\DMMD$ consists in the need to use the adversarial training, which makes the hyperparameters choice tricky. $\DMMD$ on the other hand, relies on a simple non-adversarial training procedure from Algortihm~\ref{alg:mmd_training}.
Finally, we see that DDPM performs better than $\DMMD$. This is not surprising, since both, U-Net architecture and forward diffusion process~\eqref{eq:forward_diffusion} were optimized for DDPM performance.
% We hypothesize that further investigation of impact of these design choices will allow us to make the performance gap smaller.
Nevertheless, $\DMMD$ demonstrates strong empirical performance as a discriminator flow method trained without adversarial training.
The samples from our method are provided in Appendix~\ref{app:samples}.

% \paragraph{Number of particles.} In Table~\ref{table:number_of_particles} we report performance of $\DMMD$ depending on the number of particles $\Npart$ at sampling time. As expected as the number of particles increases, the FID score decreases, but the overall performance is sensitive to the number of particles. This motivates the approximate sampling procedure from Section~\ref{sec:approximate_sampling}.
% \begin{table}[!htbp]
% \small
% \centering
% \caption{\textbf{Number of particles ablation}, FIDs on CIFAR10.}
% \label{table:number_of_particles}
% \begin{tabular}{cccc}
% \toprule
% $\Npart = 50$ & $\Npart = 100$ & $\Npart = 200$ \\
% \midrule
% 9.76 & 8.55 & 8.31 \\
% \end{tabular}
% \end{table}

\paragraph{Approximate sampling.} We run approximate $\MMD$ gradient flow \eqref{eq:approximate_mmd_grad_flow} with the same discriminator as for $\DMMD$. We call this variant $a$-$\DMMD$, where $a$ stands for \emph{approximate}. On top of that, we use denoising procedure described in Section~\ref{sec:sampling}. Starting from the samples given by $a$-$\DMMD$, we do $2$ gradient flow steps with higher learning rate using either approximate gradient flow~\eqref{eq:approximate_mmd_grad_flow}, which we call $a$-$\DMMD$-$a$, or exact gradient flow~\eqref{eq:adaptive_mmd_grad_flow} applied to a single particle, which we call $a$-$\DMMD$-$e$, $e$ stands for \emph{exact}. On top of that, we apply the denoising to $\DMMD$, which we call $\DMMD$-$e$.
Results are provided in Table~\ref{table:number_of_particles}. We observe that $a$-$\DMMD$ performs worse than $\DMMD$, which is as expected. Applying a denoising step improves performance of $a$-$\DMMD$, bringing it closer to $\DMMD$.
This suggests that the approximation~\eqref{eq:approximation} moves the particles close to the target distribution; but once close to the target, a more refined procedure is required. By contrast, we see that denoising helps $\DMMD$ only marginally. This suggests that the \emph{exact} noise-conditional witness function~\eqref{eq:witness_function_empirical} accurately captures fine detais close to the target distribution.
\begin{table}[!htbp]
\small
\centering
\caption{\textbf{Approximate sampling} performance on CIFAR10.} 
\label{table:approximate}
\begin{tabular}{llll}
\toprule
% $\DMMD$ & $\DMMD$-$e$ & $a$-$\DMMD$ & $a$-$\DMMD$-$a$ & $a$-$\DMMD$-$e$ \\
% 8.31 & 8.21 & 
Method & FID & Inception Score & NFE \\
\midrule
$\DMMD$ & 8.31 & 9.09 & 100 \\
$\DMMD$-$e$ & \textbf{8.21} & 8.99 & 102 \\
\midrule
$a$-$\DMMD$ & 24.86 & 9.10 & 50 \\
$a$-$\DMMD$-$e$ & \textbf{9.185} & 8.70 & 52 \\
$a$-$\DMMD$-$a$ & 11.22 & 9.00 & 52 \\
\end{tabular}
\vspace{-2.5mm}
\end{table}
\subsection{Results on CELEB-A, LSUN-Church and MNIST}
\label{sec:additional_results}

Besides CIFAR-10, we study the performance of $\DMMD$ on MNIST~\citep{lecun1998gradient},  CELEB-A (64x64~\citep{liu2015deep} and  LSUN-Church (64x64)~\citep{yu2016lsun}. For MNIST and CELEB-A, we consider the same splits and evaluation regime as in \citep{franceschi2023unifying}. For LSUN Church, the splits and the evaluation regime are taken from \citep{ho2020denoising}. For more details, see Appendix~\ref{app:additional_datasets}. The results are provided in Table~\ref{table:additional_fids}. In addition to $\DMMD$, we report the performance of \emph{Discriminator flow} baseline from \citep{franceschi2023unifying} with numbers taken from the corresponding paper. We see that $\DMMD$ performance is significantly better compared to the discriminator flow, which is consistent with our findings on CIFAR-10. The corresponding samples are provided in Appendix~\ref{app:additional_datasets_samples}.

\begin{table}[!htbp]
\small
\centering
\caption{\textbf{Unconditional image generation on additional datasets}. The metric used is FID. The number of gradient flow steps for $\DMMD$ is 100.} 
\label{table:additional_fids}
\begin{tabular}{lcc}
\toprule
Dataset & $\DMMD$ & Disc. flow~\citep{franceschi2023unifying} \\
\midrule
MNIST & \textbf{3.0} & 4.0 \\
CELEB-A & \textbf{8.3} & 41.0 \\
LSUN & \textbf{6.1} & - \\
\end{tabular}
\end{table}

\section{Conclusion}
\label{sec:conclusion}

In this paper we have presented a method to train a noise conditional discriminator without adversarial training, using a forward diffusion process.
% At training time our method mimics GAN training.
% by using forward
% since a noise-conditional discriminator for high level of noise corresponds to an undertrained GAN discriminator. On the other hand, for low level of noise it corresponds to a well-trained GAN discriminator.
We use this noise conditional discriminator to generate samples using a noise adaptive MMD gradient flow. We provide  theoretical insight into why an adaptive gradient flow can provide faster convergence than the non-adaptive variant. We demonstrate strong empirical performance of our method on uncoditional image generation of CIFAR10,
%\arxiv{
as well as on  additional, similar image datasets.
%}
We propose a scalable approximation of our approach which has close to the original empirical performance.
%\arxiv{
%\final{Finally, we show that idea described for $\DMMD$ leads to reasonable empirical performance when regularized lower bound on KL-divergence (which is an instance of f-divergence) is used in place of the MMD.}
%}

% \arxiv{
A number of questions remain open for future work.  The empirical performance of $\DMMD$ will be of interest in regimes where diffusion \final{models could be ill-behaved, such as in generative modeling on Riemannian manifolds; as well as on larger datasets such as ImageNet.}  $\DMMD$ provides a way of training a discriminator, which may be applicable in other areas where a domain-adaptive discriminator might be required.
%The performance of f-divergence-based discriminators still somewhat lags the corresponding 
Finally, it will be of interest to establish theoretical foundations for $\DMMD$ in general settings, and to derive convergence results for the associated flow.

%%%%% AG: I removed the below. Some parts were moved to the conclusions  that were already there. For others, eg the f-GAN, I think it's not useful for us to bring them up again in the conclusions, since the results weren't that great...

%\arxiv{Finally, we have conducted experiments on using  f-divergence-based discriminators, with performance close to an f-GAN. In the future work, we will focus on scaling $\DMMD$ to large datasets such as ImageNet, as well as to study its performance in settings which go beyond generative modelling.}

% In future work, we will focus on scaling $\DMMD$ to large datasets, we will focus on building theoretical foundations for $\DMMD$ in general setting and derive convergence results. 

% In the future work, we will study the empirical performance of $\DMMD$

% In future work, 

% While the present work has focused on the MMD discriminator, we believe that the approach will extend to f-GAN discriminators, which is an interesting topic for future study. 
% In the future work, we will explore ways of improving the method, such as using guided sampling techniques.

% In the future work, we will further scale the method to larger datasets. In another line of work, we are also interested in developing the theoretical properties of such flows, extending the work of~\citep{arbel2019maximum}.

\section{Impact statement}

In  recent years, generative modeling has undergone a period of rapid and transformative progress. Diffusion models play a pivotal role in driving this revolution with applications ranging from image synthesis to protein modeling. Recent works have sought ways to combined the strengths of both GANs and diffusion models. In this spirit, we propose in this work a new methodology for generative modeling, demonstrating that noise-adaptive discriminators may be defined to produce such flows without adversarial training. We believe that our work can shed light on the key components of these successful generative models and pave the way for further improvements of discriminative flow methods. 

%better understanding theoretical convergence properties of this approach. 

% In this paper we have presented a novel method of training a noise conditional discriminator without adversarial training using forward diffusion process. This forward diffusion process mimics the insights from GAN training such that a sampled data with high level of noise corresponds to a weak generator whereas the one with low level of noise corresponds to a strong one. We then later show that such trained discriminator could be used to generate samples using noise adaptive gradient flow. As it stands now, our method is the first gradient flow method which demonstrates strong performance on image generations tasks. Understanding theoretical properties of our approach as well as possible empirical improvements is left for future work.

%\section*{Acknowledgements}

%\textbf{Do not} include acknowledgements in the initial version of
% the paper submitted for blind review.

\nocite{langley00}

\bibliography{main}
\bibliographystyle{icml2024}

\newpage
\appendix
\onecolumn

\section{Organization of the supplementary material}

In Section~\ref{app:dmmd_training_and_sampling}, we describe in details the training and sampling procedures for $\DMMD$. In Section~\ref{app:2d_toy_datasets}, we describe more details for the 2d experiments. In Section~\ref{app:d_kale_flow}, we provide more details about $\DKALE$-Flow method. In Section~\ref{app:experiments}, we provide experimental details for the image datasets. In Section~\ref{section:proof_theoretical_justification}, we provide proof for the theoretical results described in Section~\ref{sec:theoretical_justification} from the main section of the paper. Finally, in Section~\ref{app:all_samples} we present the samples from $\DMMD$ on different image datasets.

\section{$\DMMD$ training and sampling}
\label{app:dmmd_training_and_sampling}
\subsection{MMD discriminator}
Let $\mathcal{X} \subset \rset^{D}$ and $\mathcal{P}(\mathcal{X})$ be the set of probability distributions defined on $\mathcal{X}$. Let $P \in \mathcal{P}(\mathcal{X})$ be the \textit{target} or data distribution and $Q_{\psi} \in \mathcal{P}(\mathcal{X})$ be a distribution associated with a \emph{generator} parameterized by $\psi \in \rset^{L}$. Let $\mathcal{H}$ be Reproducing Kernel Hilbert Space (RKHS), see \citep{learning_with_kernels} for details, for some kernel $k : \mathcal{X} \times \mathcal{X} \rightarrow \rset$.
Maximum Mean Discrepancy (MMD)~\citep{gretton12a} between $Q_{\psi}$ and $P$ is defined as $\MMD(Q_{\psi}, P) = \sup_{f \in \mathcal{H}} \{\mathbb{E}_{Q_{\psi}} [f(X)] - \mathbb{E}_{P} [f(X)] \}$.
Given $X^N = \{x_{i}\}_{i=1}^{N} \sim Q_{\psi}^{\otimes N}$ and $Y^M = \{y_{i}\}_{i=1}^{M} \sim P^{\otimes M}$, an unbiased estimate of $\MMD^2$~\citep{gretton12a} is given by
\begin{align}
    \label{eq:unbiased_mmd_app}
    &\textstyle
    \MMD^2_u[X^N, Y^M] = \tfrac{1}{N (N-1)} \sum_{i \neq j}^{N} k(x_i, x_j) + \\
    &\textstyle
    \tfrac{1}{M (M-1)} \sum_{i \neq j}^{M} k(y_i, y_j) -
    \tfrac{2}{N M} \sum_{i=1}^{N} \sum_{j=1}^{M} k(x_i, y_j) . \nonumber
\end{align}
In MMD GAN~\citep{binkowski2021demystifying,li2017mmd}, the kernel in the objective~\eqref{eq:unbiased_mmd_app} is given as
\begin{equation}
    \label{eq:mmd_gan_kernel_app}
    \textstyle
    k(x, y) = \kbase(\phi(x; \theta), \phi(y; \theta)),
\end{equation}
where $\kbase$ is a base kernel and $\phi(\cdot; \theta): \mathcal{X} \rightarrow \rset^{K}$ are neural networks \emph{discriminator} features with parameters $\theta \in \rset^{H}$. We use the modified notation of $\MMD^2_u[X^N, Y^M; \theta]$ for equation~\eqref{eq:unbiased_mmd_app} to highlight the functional dependence on the discriminator parameters. $\MMD$ is an instance of Integral Probability Metric (IPM) (see \citep{arjovsky2017wasserstein}) which is well defined on distributions with disjoint support unlike f-divergences~\citep{nowozin2016fgan}. An advantage of using MMD over other IPMs (see for example, Wasserstein GAN~\citep{arjovsky2017wasserstein}) is the flexibility to choose a kernel $k$. Another form of $\MMD$ is expressed as a norm of a \emph{witness function}
\begin{equation}
    \MMD(Q_{\psi}, P) = \sup_{f \in \mathcal{H}} \{\mathbb{E}_{Q_{\psi}} [f(X)] - \mathbb{E}_{P} [f(X)] \} = \|f_{Q_{\psi}, P} \|_{\mathcal{H}},
\end{equation}
where the witness function $f_{Q_{\psi}, P}$ is given as
\begin{equation}
    \label{eq:witness_function_app}
    f_{Q_{\psi}, P}(z) = \int k(x, z) d Q_{\psi} - \int k(y, z) d P(y)
\end{equation}
Given two sets of samples $X^N = \{x_i\}_{i=1}^{N} \sim Q_{\psi}^{\otimes N}$ and $Y^M = \{y_i\}_{i=1}^{M} \sim P^{\otimes M}$, and the kernel \eqref{eq:mmd_gan_kernel_app}, the empirical witness function is given as
\begin{equation}
    \hat{f}_{Q_{\psi}, P}(z) = \frac{1}{N} \sum_{i=1}^{N} \kbase(\phi(x_i; \theta), \phi(z; \theta)) - \frac{1}{M} \sum_{j=1}^M \kbase(\phi(y_j; \theta), \phi(z; \theta))
\end{equation}
The $\ell_2$ penalty~\citep{binkowski2021demystifying} is defined as
\begin{equation}
    \label{eq:l2penalty_app}
    \mathcal{L}_{\ell_2}(\theta) = \frac{1}{N} \sum_{i=1}^N \|\phi(x_i; \theta)\|^2_2 + \frac{1}{N} \sum_{i=1}^N \|\phi(y_i; \theta)\|^2_2  
\end{equation}
Assuming that $M=N$ and following \citep{binkowski2021demystifying,gulrajani2017improved}, for $\alpha_i \sim U[0,1]$, where $U[0,1]$ is a uniform distribution on $[0,1]$, we construct $z_i = x_i \alpha_i + (1 - \alpha) y_i$ for all $i=1,\ldots,N$. Then, the gradient penalty~\citep{binkowski2021demystifying,gulrajani2017improved} is defined as
\begin{equation}
    \label{eq:gp_app}
    \mathcal{L}_{\nabla}(\theta) = \frac{1}{N} \sum_{i=1}^{N} (\| \nabla \hat{f}_{Q_{\psi}, P} (z_i) \|_2 - 1)^2
\end{equation}

We denote by $\mathcal{L}(\theta)$ the $\MMD$ discriminator loss given as
\begin{align}
    \mathcal{L}(\theta) = -\MMD^2_u[X^N, Y^M; \theta] = \tfrac{1}{N (N-1)} \sum_{i \neq j}^{N} \kbase(\phi(x_i; \theta), \phi(x_j; \theta)) + \tfrac{1}{M (M-1)} \sum_{i \neq j}^{M} \kbase(\phi(y_i; \theta), \phi(y_j; \theta)) \\
    &\textstyle
    - \tfrac{2}{N M} \sum_{i=1}^{N} \sum_{j=1}^{M} \kbase(\phi(x_i; \theta), \phi(y_j; \theta))
\end{align}

Then, the total loss for the discriminator on the two samples of data assuming that $N=M$ is given as
\begin{equation}
    \calLtot(\theta) = \calL(\theta) + \lambda_{\nabla} \calL_{\nabla}(\theta) + \lambda_{\ell_2} \calL_{\ell_2}(\theta),
\end{equation}
for some constants $\lambda_{\nabla} \geq 0$ and $\lambda_{\ell_2} \geq 0$.

\subsection{Noise-dependent $\MMD$}
\label{app:noise_dependent_mmd}

In Section~\ref{sec:method}, we describe the approach to train $\MMD$ discriminator from forward diffusion using noise-dependent discriminators. For that, we assume that we are given a noise level $t \sim U[0,1]$ where $U[0,1]$ is a uniform distribution on $[0,1]$, and a set of clean data $X^N = \{x^i\}_{i=1}^{N} \sim P^{\otimes N}$. Then we produce a set of noisy samples $x^i_t$ using forward diffusion process~\eqref{eq:forward_diffusion}. We denote these samples by $X^N_t =\{ x^i_t \}_{i=1}^{N}$. We define noise conditional kernel
\begin{equation}
    \label{eq:noise_cond_kernel_app}
    k(x, y; t, \theta) = \kbase(\phi(x, t; \theta), \phi(y, t; \theta)),
\end{equation}
with noise conditional features $\phi(x, t; \theta)$. This allows us to define the noise conditional discriminator loss
\begin{align}
    \label{eq:noise_cond_mmd_app}
    \mathcal{L}(\theta, t) = -\MMD^2_u[X^N, X^N_t, t, \theta] = \tfrac{1}{N (N-1)} \sum_{i \neq j}^{N} \kbase(\phi(x^i_t; t, \theta), \phi(x^j_t; t, \theta)) + \tfrac{1}{N (N-1)} \sum_{i \neq j}^{N} \kbase(\phi(x^i; t, \theta), \phi(x^j; t, \theta)) \\
    &\textstyle
     - \tfrac{2}{N^2} \sum_{i=1}^{N} \sum_{j=1}^{N} \kbase(\phi(x^i; t, \theta), \phi(x^j_t; t, \theta))
\end{align}
The noise conditional $\ell_2$ penalty is given as 
\begin{equation}
    \mathcal{L}_{\ell_2}(\theta, t) = \frac{1}{N} \sum_{i=1}^N \|\phi(x^i_t; t, \theta)\|^2_2 + \frac{1}{N} \sum_{i=1}^N \|\phi(x^i; t, \theta)\|^2_2  
\end{equation}
The noise conditional gradient penalty is given as 
\begin{equation}
    \mathcal{L}_{\nabla}(\theta, t) = \frac{1}{N} \sum_{i=1}^{N} (\| \nabla \hat{f}_{P, t} (z_i) \|_2 - 1)^2,
\end{equation}
where $z_i = \alpha_i x^i_t + (1-\alpha_i) x^i$ for $\alpha_i \sim U[0,1]$ and the noise conditional witness function
\begin{equation}
    \label{eq:noise_dependent_witness_app}
    \hat{f}_{P, t}(z) = \frac{1}{N} \sum_{i=1}^{N} \kbase(\phi(x^t_i; t, \theta), \phi(z; \theta)) - \frac{1}{N} \sum_{j=1}^N \kbase(\phi(x_i; t, \theta), \phi(z; \theta))
\end{equation}
Therefore, the total noise conditional loss is given as
\begin{equation}
    \label{eq:noise_conditional_total_loss_app}
    \calLtot(\theta, t) = \calL(\theta, t) + \lambda_{\nabla} \calL_{\nabla}(\theta, t) + \lambda_{\ell_2} \calL_{\ell_2}(\theta, t),
\end{equation}
for some constants $\lambda_{\nabla} \geq 0$ and $\lambda_{\ell_2} \geq 0$.

\subsection{Linear kernel for scalable $\MMD$}
\label{app:linear_optimization}

Computational complexity of \eqref{eq:noise_conditional_total_loss_app} is $O(N^2)$. Here, we assume that the base kernel is linear, i.e.
\begin{equation}
    \kbase(x, y) = \langle x, y \rangle
\end{equation}
This allows us to simplify the $\MMD$ computation~\eqref{eq:noise_cond_mmd_app} as
\begin{multline}
    \label{eq:mmd_optimized_app}
    \MMD^2_u[X^N, X^N_t, t, \theta] =
    \frac{1}{N (N-1) } \left( \bar{\phi_t}(X^N_t)^T \bar{\phi_t}(X^N_t) - \bar{\| \phi_t \|^2}^N (X_t) \right) +
    \frac{1}{N (N-1) } \left( \bar{\phi_t}(X^N)^T \bar{\phi_t}(X^N) - \bar{\| \phi_t \|^2}^N (Y) \right) \\
    - \frac{2}{N N}(\bar{\phi_t} (X^N_t)) ^T \bar{\phi_t}(X^N),
\end{multline}
where
\begin{align}
    \bar{\phi_t}(X^N_t) = \sum_{i=1}^N \phi(x^i_t; \theta_t) \\
    \bar{\phi_t}(X^N) = \sum_{j=1}^N \phi(x^i; \theta_t) \\
    \bar{\| \phi_t \|^2} (X^N_t) = \sum_{i=1}^N \| \phi(x^i_t; \theta_t) \|^2 \\
    \bar{\| \phi_t \|^2} (X^N) = \sum_{j=1}^N \| \phi(x^i; \theta_t) \|^2
\end{align}

Therefore we can pre-compute quantities $\bar{\phi_t}(X^N_t),\bar{\phi_t}(X^N),\bar{\| \phi_t \|^2} (X^N_t),\bar{\| \phi_t \|^2} (X^N)$ which takes $O(N)$ and compute $\MMD^2_u[X^N, X^N_t, t, \theta]$ in $O(1)$ time. This also leads $O(1)$ computation complexity for $\calL_{\ell_2}$ and $O(N)$ complexity for $\calL_{\nabla}$. This means that we simplify the computational complexity to $O(N)$ from $O(N^2)$.

At sampling, following \eqref{eq:adaptive_mmd_grad_flow} requires to compute  the witness function \eqref{eq:noise_dependent_witness_app} for each particle, which for a general kernel takes $O(N^2)$ in total. Using the linear kernel above, simplifies the complexity of the witness as follows
\begin{equation}
  \hat{f}_{P, t}(z) = \langle \bar{\phi_t}(Z^N) - \bar{\phi_t}(X^N), \phi(z; \theta) \rangle,
\end{equation}
where $Z^N$ is a set of $N$ noisy particles. We can precompute $\bar{\phi_t}(Z^N)$ in $O(N)$ time. Therefore one iteration of a witness function will take $O(1)$ time and for $N$ noisy particles it makes $O(N)$.

\subsection{Approximate sampling procedure}
\label{app:approximate_sampling}

In this section we provide an algorithm for the approximate sampling procedure. The only change with the original Algorithm~\ref{alg:sampling_from_gradient_flow} is the approximate witness function
\begin{equation}
    \hat{f}^\star_{P_t, P}(z) = \langle \phi(z, t; \theta^\star), \bar{\phi}(X_t, t, \theta^\star) - \bar{\phi}(X_0, t, \theta^\star) \rangle,
\end{equation}
where
\begin{align}
    \label{eq:precomputed_approximate_app}
    \bar{\phi}(X_0, t, \theta^\star = \frac{1}{N} \sum_{i=1}^{N} \phi(x^i_0, t; \theta^\star) \\
    \bar{\phi}(X_t, t, \theta^\star = \frac{1}{N} \sum_{i=1}^{N} \phi(x^i_t, t; \theta^\star)
\end{align}
Here $x^i_0,i=1,\ldots,N$ correspond to the whole training set of clean samples and $x^i_t,i=1,\ldots,$ correspond to the noisy version of these clean samples produced by the forward diffusion process~\ref{eq:forward_diffusion} for a given noise level $t$. These features can be precomputed once for every noise level $t$. The resulting algorithm is given in Algorithm~\eqref{alg:sampling_from_gradient_flow_approximate_app}. Another crucial difference with the original algorithm is the ability to run it for each particle $Z$ independently.

\begin{algorithm}[tb]
   \caption{Approximate noise-adaptive $\MMD$ gradient flow for a single particle}
   \label{alg:sampling_from_gradient_flow_approximate_app}
\begin{algorithmic}
   \STATE {\bfseries Inputs:} $T$ is the number of noise levels
   \STATE $\tmax, \tmin$ are maximum and minimum noise levels
   \STATE $\Nfixed$ is the number of gradient flow steps per noise level
   \STATE $\eta > 0$ is the gradient flow learning rate
   \STATE $\bar{\phi}(X_0, t, \theta^\star)$ - precomputed clean features for all $t=1,\ldots,T$ with \eqref{eq:precomputed_approximate_app}
   \STATE $\bar{\phi}(X_t, t, \theta^\star)$ - precomputed noisy features for all $t=1,\ldots,T$ with \eqref{eq:precomputed_approximate_app}
   \STATE {\bfseries Steps:} Sample initial noisy particle $Z \sim \mathrm{N}(0,\Id)$
   \FOR{$i=T$ {\bfseries to} $0$}
     \STATE Set the noise level $t = i \Delta t$ and $Z^t_0 = Z$
     \FOR{$n=0$ {\bfseries to} $\Nfixed - 1$}
        \STATE $Z^{t}_{n+1} = Z^{t}_{n} - \eta \langle \nabla_z \phi(Z^{t}_{n}, t; \theta^\star), \bar{\phi}(X_t, t, \theta^\star) - \bar{\phi}(X_0, t, \theta^\star) \rangle$
     \ENDFOR
     \STATE Set $Z = Z^{t}_{N}$
   \ENDFOR
   \STATE Output $Z$
\end{algorithmic}
\end{algorithm}

\section{Toy 2-D datasets experiments}
\label{app:2d_toy_datasets}

For the 2-D experiments, we train $\DMMD$ using Algorithm~\eqref{alg:mmd_training} for $\Niter=50000$ steps with a batch size of $B=256$ and a number of noise levels per batch equal to $\Nnoise=128$. The Gradient penalty constant $\lambda_{\nabla} = 0.1$ whereas the $\ell_2$ penalty is not used. To learn noise-conditional $\MMD$ for $\DMMD$, we use a $4$-layers MLP $g(t; \theta)$ with ReLU activation to encode $\sigma(t; \theta) = \sigma_{\text{min}} + \text{ReLU}(g(t; \theta))$ with $\sigma_{\text{min}}=0.001$, which ensures $\sigma(t; \theta) > 0$. The MLP layers have the architecture of $[64, 32, 16, 1]$. Before passing the noise level $t \in [0, 1]$ to the MLP, we use sinusoidal embedding similar to the one used in \citep{ho2020denoising}, which produces a feature vector of size $1024$. The forward diffusion process from \citep{ho2020denoising} have modified parameters such that corresponding $\beta_1=10^-4,\beta_T=0.0002$. On top of that, we discretize the corresponding process using only $1000$ possible noise levels, with $\tmin = 0.05$ and $\tmax=1.0$. At sampling time for the algorithm~\ref{alg:sampling_from_gradient_flow}, we use $\tmin=0.05,\tmax=1.0$, $\Nfixed=10$ and $T=100$. The learning rate $\eta = 1.0$. As basleines, we consider $\MMD$-GAN with a generator parameterised by a $3$-layer MLP with ELU activations. The architecture of the MLP is $[256, 256, 2]$. The initial noise for the generator is produced from a uniform distribution $U[-1,1]$ with a dimensionality of $128$. The gradient penalty coefficient equals to $0.1$. As for the discriminator, the only learnable parameter is $\sigma$. We train $\MMD$-GAN for $250000$ iterations with a batch size of $B=256$. Other variants of $\MMD$ gradient flow use the same sampling parameters as $\DMMD$.

\section{D-KALE-flow}
\label{app:d_kale_flow}

In this section, we describe the $\DKALE$-flow algorithm mentioned in Section~\ref{sec:f_divergences}. Let $\mathcal{X} \subset \rset^{D}$ and $\mathcal{P}(\mathcal{X})$ be the set of probability distributions defined on $\mathcal{X}$. Let $P \in \mathcal{P}(\mathcal{X})$ be the \textit{target} or data distribution and $Q \in \mathcal{P}(\mathcal{X})$ be some distribution. The KALE objective (see \citep{glaser2021kale}) is defined as
\begin{equation}
    \label{eq:kale_divergence_app}
    KALE(Q, P | \lambda) = (1+\lambda) \max_{h \in \mathcal{H}} \{ 1 + \int h dQ - \int e^{h} dP  - \frac{\lambda}{2} ||h||^2_{\mathcal{H}} \},
\end{equation}
where $\lambda \geq 0$ is a positive constant and $\mathcal{H}$ is the RKHS with a kernel $k$. In practice, KALE divergence \eqref{eq:kale_divergence_app} can be replaced by a corresponding parametric objective
\begin{equation}
    \label{eq:parametric_kale_app}
    KALE(Q, P | \lambda, \theta, \alpha) = (1+\lambda) \left( \int h(X; \theta, \alpha) dQ(X) - \int e^{h(Y; \theta, \alpha)} dP(Y) - \frac{\lambda}{2} ||\alpha||^2_{2} \right),
\end{equation}
where
\begin{equation}
    h(X; \theta, \alpha) = \phi(X; \theta)^T \alpha,
\end{equation}
with $\phi(X; \theta) \in \rset^{D}$ and $\alpha \in \rset^{D}$. The objective function \eqref{eq:parametric_kale_app} can then be maximized with respect to $\theta$ and $\alpha$ for given $Q$ and $P$. Similar to $\DMMD$, we consider a noise-conditional witness function
\begin{equation}
    h(x; t, \theta, \alpha, \psi) = \phi(x; t, \theta)^T \alpha({t}; \psi)
\end{equation}
From here, the noise-conditional KALE objective is given as
\begin{equation}
    \label{eq:noise_conditional_kale_app}
    \calL(\theta, \psi, t | \lambda) = KALE(P_{t}, P | \lambda, \theta, \alpha),
\end{equation}
where $P_t$ is the distribution corresponding to a forward diffusion process, see Section~\ref{sec:mmd_training}. Then, the total noise-conditional objective is given as
\begin{equation}
    \calLtot(\theta, \psi, t | \lambda) = \calL(\theta, \psi, t | \lambda) + \lambda_{\nabla} \calL_{\nabla}(\theta, \psi, t) + \lambda_{\ell_{2}} \calL_{\ell_{2}}(\theta, t),
\end{equation}
where gradient penalty has similar form to WGAN-GP~\citep{gulrajani2017improved}
\begin{equation}
    \calL_{\nabla}(\theta, \psi, t) = \mathbb{E}_{Z}(||\nabla_{Z} h(Z; t, \theta, \alpha, \psi)||_{2} - 1)^2,
\end{equation}
where $Z = \beta X + (1-\beta) Y$, $\beta \sim U[0,1]$, $X \sim P(X)$ and $Y \sim P(Y)$. The l2 penalty is given as
\begin{equation}
    \calL_{\ell_{2}}(\theta, t) = \frac{1}{2} \left( \mathbb{E}_{X \sim P(X)} ||\phi(X; t, \theta)||^2 + \mathbb{E}_{Y \sim P(Y)} ||\phi(Y; t, \theta)||^2\right)
\end{equation}
Therefore, the final objective function to train the discriminator is
\begin{equation}
    \calLtot(\theta, \psi | \lambda) = \mathbb{E}_{t \sim U[0, 1]} \left[ \calLtot(\theta, \psi, t | \lambda) \right]
\end{equation}
This objective function depends on RKHS regularization $\lambda$, on gradient penalty regularization $\lambda_{\nabla}$ and on l2-penalty regularization $\lambda_{\ell_{2}}$. Unlike in $\DMMD$, we do not use an explicit form for the witness function and do not use the RKHS parameterisation. On one hand, this allows us to have a more scalable approach, since we can compute $\KALE$ and the witness function for each individual particle. On the other hand, the explicit form of the witness function in $\DMMD$ introduces beneficial inductive bias. In $\DMMD$, when we train the discriminator, we only learn the kernel features, i.e. corresponding RKHS. In $\mathrm{D}$-$\KALE$, we need to learn both, the kernel features $\phi(x; t, \theta)$ as well as the RKHS projections $\alpha(t; \psi)$. This makes the learning problem for $\mathrm{D}$-$\KALE$ more complex. The corresponding noise adaptive gradient flow for $\KALE$ divergence is described in Algorithm~\ref{alg:d_kale_flow}. An advantage over original $\DMMD$ gradient flow is the ability to run this flow individually for each particle.

\begin{algorithm}[tb]
   \caption{Noise-adaptive $\KALE$ flow for single particle}
   \label{alg:d_kale_flow}
\begin{algorithmic}
   \STATE {\bfseries Inputs:} $T$ is the number of noise levels
   \STATE $\tmax, \tmin$ are maximum and minimum noise levels
   \STATE $\Nfixed$ is the number of gradient flow steps per noise level
   \STATE $\eta > 0$ is the gradient flow learning rate
   \STATE Trained witness function $h(\cdot; t, \theta^\star, \psi^\star)$
   \STATE {\bfseries Steps:} Sample initial noisy particle $Z \sim \mathrm{N}(0,\Id)$
   \STATE Set $\Delta t = (\tmax - \tmin) / T$
   \FOR{$i=T$ {\bfseries to} $0$}
     \STATE Set the noise level $t = \tmin + i \Delta t$ and $Z^t_0 = Z$
     \FOR{$n=0$ {\bfseries to} $\Nfixed - 1$}
        \STATE  $Z^{t}_{n+1} = Z^{t}_{n} - \eta \nabla h(Z^t_{n}; t, \theta^\star, \psi^\star)$
     \ENDFOR
     \STATE Set $Z = Z^{t}_{N}$
   \ENDFOR
   \STATE Output $Z$
\end{algorithmic}
\end{algorithm}

\section{Image generation experiments}
\label{app:experiments}

For the image experiments, we use CIFAR10~\citep{krizhevsky2009learning} dataset. We use the same forward diffusion process as in \citep{ho2020denoising}. As a Neural Network backbone, we use U-Net~\citep{ronneberger2015unet} with a slightly modified architecture from \citep{ho2020denoising}. Our neural network architecture follows the backbone used in \citep{ho2020denoising}. On top of that we output the intermediate features at four levels -- before down sampling, after down-sampling, before upsampling and a final layer. Each of these feature vectors is processed by a group normalization, the activation function and a linear layer producing an output vector of size $32$. To produce the output of a discriminator features, these four feature vectors are concatenated to produce a final output feature vector of size $128$. The noise level time is processed via sinusoidal time embedding similar to \citep{ho2020denoising}. We use a dropout of $0.2$. $\DMMD$ is trained for $\Niter=250000$ iterations with a batch size $B=64$ with number $\Nnoise=16$ of noise levels per batch. We use a gradient penalty $\lambda_{\nabla}=1.0$ and $\ell_{2}$ regularisation strength $\lambda_{\ell_2}=0.1$. As evaluation metrics, we use FID~\citep{heusel2018gans} and Inception Score~\citep{salimans2016improved} using the same evaluation regime as in \citep{ho2020denoising}. To select hyperparameters and track performance during training, we use FID evaluated on a subset of $1024$ images from a training set of CIFAR10.

For CIFAR10, we use random flip data augmentation.

In $\DMMD$ we have two sets of hyperparameters, one is used for training in Algorithm~\ref{alg:mmd_training} and one is used for sampling in Algorithm~\ref{alg:sampling_from_gradient_flow}. During training, we fix the sampling parameters and always use these to select the best set of training time hyperparameters. We use $\eta = 0.1$ gradient flow learning rate, $T=10$ number of noise levels, $\Npart=200$ number of noisy particles, $\Nfixed=5$ number of gradient flow steps per noise level, $\tmin=0.001$ and $\tmax=1 - 0.001$. We use a batch of $400$ clean particles during training. For hyperparameters, we do a grid search for $\lambda_{\nabla} \in \{ 0, 0.001, 0.01, 0.1, 1.0, 10.0 \}$, for $\lambda_{\ell_2} \in \{ 0, 0.001, 0.01, 0.1, 1.0, 10.0 \}$, for dropout rate $\{ 0, 0.1, 0.2, 0.3 \}$, for batch size $\{ 16, 32, 64 \}$. To train the model, we use the same optimization procedure as in \citep{ho2020denoising}, notably Adam~\citep{kingma2017adam} optimizer with a learning rate $0.0001$. We also swept over the the dimensionality of the output layer $32, 64, 128$, such that each of four feature vectors got the equal dimension. Moreover, we swept over the number of channels for U-Net $\{ 32, 64, 128 \}$ (the original one was $32$) and we found that $128$ gave us the best empirical results.

After having selected the training-time hyperparameters and having trained the model, we run a sweep for the sampling time hyperparameters over $\eta \in \{ 1, 0.5, 0.1, 0.04, 0.01 \}$, $T \in \{1, 5, 10, 50 \}$, $\Nfixed \in \{1, 5, 10, 50 \}$, $\tmin \in \{0.001, 0.01, 0.1, 0.2 \}$, $\tmax \in \{0.9, 1 - 0.001 \}$.
% The hyperparameters are selected using FID evaluated on a subset of CIFAR10.
We found that the best hyperparameters for $\DMMD$ were $\eta=0.1$, $\Nfixed=10$, $T=10$, $\tmin=0.1$ and $\tmax=0.9$. On top of that, we ran a variant for $\DMMD$ with $T=50$ and $\Nfixed=5$.

For $a$-$\DMMD$ method, we used the same pretrained discriminator as for $\DMMD$ but we did an additional sweep over sampling time hyperparameters, because in principle these could be different. We found that the best hyperparameters for $a$-$\DMMD$ are $\eta=0.04$, $\tmin=0.2$, $\tmax=0.9$, $T=5$, $\Nfixed=10$.

For the denoising step, see Table~\ref{table:approximate}, for $\DMMD$-$e$, we used 2 steps of $\DMMD$ gradient flow with a higher learning rate $\eta^\star=0.5$ with $\tmax=0.1$ and $\tmin=0.001$. For $a$-$\DMMD$-$e$, we used 2 steps of $\DMMD$ gradient flow with a higher learning rate of $\eta^\star=0.5$ with $\tmax=0.2$ and $\tmin=0.001$. For $a$-$\DMMD$-$e$, we used 2 steps of $\DMMD$ gradient flow with a higher learning rate of $\eta^\star=0.1$ with $\tmax=0.2$ and $\tmin=0.001$. The only parameter we swept over in this experiment was this higher learning rate $\eta^\star$.

After having found the best hyperparameters for sampling, we run the evaluation to compute FID on the whole CIFAR10 dataset using the same regime as described in ~\citep{ho2020denoising}.

For $\MMD$-GAN experiment, we use the same discriminator as for $\DMMD$ but on top of that we train a generator using the same U-net architecture as for $\DMMD$ with an exception that we do not use the 4 levels of features. We use a higher gradient penalty of $\lambda_{\nabla}=10$ and the same $\ell_2$ penalty $\lambda_{\ell_2}=0.1$. We use a batch size of $B=64$ and the same learning rate as for $\DMMD$. We use a dropout of $0.2$. We train $\MMD$-GAN for $250000$ iterations. For each generator update, we do $5$ discriminator updates, following \citep{brock2019large}.

For $\MMD$-GAN-Flow experiment, we take the pretrained discriminator from $\MMD$-GAN and run a gradient flow of type~\eqref{eq:mmd_grad_flow_new} on it, starting from a random noise sampled from a Gaussian. We swept over different parameters such as learning rate $\eta$ and number of iterations $\Niter$. We found that none of our parameters led to any reasonable performance. The results in Table~\ref{table:fids} are reported using $\eta=0.1$ and $\Niter=100$.

\subsection{Additional datasets}
\label{app:additional_datasets}

We study performance of $\DMMD$ on additional datasets, MNIST~\citep{lecun1998gradient}, on CELEB-A (64x64~\citep{liu2015deep} and on LSUN-Church (64x64)~\citep{yu2016lsun}. For MNIST and CELEB-A, we use the same training/test split as well as the evaluation protocol as in \citep{franceschi2023unifying}. For LSUN-Church, For LSUN Church, we compute FID on 50000 samples similar to DDPM~\citep{ho2020denoising}. For MNIST, we used the same hyperparameters during training and sampling as for CIFAR-10 with NFE=100, see Appendix~\ref{app:experiments}. For CELEB-A and LSUN, we ran a sweep over $\lambda_{\ell_2} \in \{ 0, 0.001, 0.01, 0.1, 1.0, 10.0 \}$ and found that $\ell_2=0.001$ led to the best results. For sampling, we used the same parameters as for CIFAR-10 with NFE=100. The reported results in Table~\ref{table:additional_fids} are given with NFE=100.

\subsection{$\mathrm{D}$-$\KALE$-flow details}
\label{app:d_kale_flow-details}
We study performance of $\mathrm{D}$-$\KALE$-flow on CIFAR10. We use the same architectural setting as in $\DMMD$ with the only difference of adding an additional mapping $\alpha(t; \psi)$ from noise level to $D$ dimensional feature vector, which is represented by a 2 layer MLP with hidden dimensionality of $64$ and GELU activation function. We use batch size $B=256$, dropout rate equal to $0.3$. For sampling time parameters during training, we use $\eta=0.5$, total number of noise levels $T=20$, and number of steps per noise level $N_{s} = 5$. At training, we sweep over RKHS regularization $\lambda \in \{0, 1, 10, 100, 500, 1000, 2000 \}$, gradient penalty $\lambda_{\nabla} \in \{ 0, 0.1, 1.0, 10.0, 50.0, 100.0, 250.0, 500.0, 1000.0 \}$, l2 penalty in $\{0, 0.1, 0.01, 0.001 \}$.

\subsection{Number of particles ablation}
\label{app:number_of_particles}

\paragraph{Number of particles.} In Table~\ref{table:number_of_particles} we report performance of $\DMMD$ depending on the number of particles $\Npart$ at sampling time. As expected as the number of particles increases, the FID score decreases, but the overall performance is sensitive to the number of particles. This motivates the approximate sampling procedure from Section~\ref{sec:fastSampling}.
\begin{table}[!htbp]
\small
\centering
\caption{\textbf{Number of particles ablation}, FIDs on CIFAR10.}
\label{table:number_of_particles}
\begin{tabular}{cccc}
\toprule
$\Npart = 50$ & $\Npart = 100$ & $\Npart = 200$ \\
\midrule
9.76 & 8.55 & 8.31 \\
\end{tabular}
\end{table}

\subsection{Results with f-divergence}
\label{app:additional_results_f_divergence}

We study performance of $\mathrm{D}$-$\KALE$-Flow described in Section~\ref{sec:f_divergences} and Appendix~\ref{app:d_kale_flow}, in the setting of unconditional image generation for CIFAR-10. We compare against a GAN baseline which uses the $\KALE$ divergence in the discriminator, but has adversarially trained generator. More details are described in Appendix~\ref{app:d_kale_flow} and  Appendix~\ref{app:d_kale_flow-details}. The results are given in Table~\ref{table:f_divergence}. We see that unlike with $\DMMD$, $\mathrm{D}$-$\KALE$-Flow achieves worse performance than corresponding $\KALE$-GAN - indicating that the inductive bias provided by the generator may be more helpful in this case - this is a topic for future study. 

%One crucial differences of the practical algorithm for $\mathrm{D}$-$\KALE$-Flow and for $\DMMD$ is the parameterisation of the witness function. In $\mathrm{D}$-$\KALE$-Flow, we allow the witness function to be flexibly learned by a Neural Network, whereas in $\DMMD$, we force the witness function to be explicitly parameterised by the kernel $k$ and the sets of particles. Such explicit parameterisation in $\DMMD$ adds a strong inductive bias and makes it easier to learn the witness function.

\begin{table}[!htbp]
\small
\centering
\caption{\textbf{Unconditional image generation on CIFAR-10} with $\KALE$-divergence. The number of gradient flow steps is $100$.} 
\label{table:f_divergence}
\begin{tabular}{lcc}
\toprule
Method & FID & Inception score \\
\midrule
$\mathrm{D}$-$\KALE$-Flow & 15.8 & 8.5 \\
$\KALE$-GAN & 12.7 & 8.7 \\
\end{tabular}
\end{table}

\section{Optimal kernel with Gaussian MMD}
\label{section:proof_theoretical_justification}

In this section, we prove the results of \Cref{sec:theoretical_justification}. We consider the following unnormalized Gaussian kernel
\begin{equation}
    k_\alpha(x,y) = \alpha^{-d} \exp[-\| x - y \|^2 / (2 \alpha^2)] . 
\end{equation}
For any $\mu \in \rset^d$ and $\sigma > 0$ we denote $\pi_{\mu, \sigma}$ the Gaussian distribution with mean $\mu$ and covariance matrix $\sigma^2 \Id$. We denote $\MMD^2_\alpha$ the $\MMD^2$ associated with $k_\alpha$. More precisely for any $\mu_1, \mu_2 \in \rset^d$ and $\sigma_1, \sigma_2 >0$ we have 
\begin{equation}
\label{eq:mmd_squared_appendix}
    \MMD^2_\alpha(\pi_{\mu_1, \sigma_1}, \pi_{\mu_2, \sigma_2}) = \mathbb{E}_{\pi_{\mu_1, \sigma_1} \otimes \pi_{\mu_1, \sigma_1}} \left[ k_\alpha(X,X') \right] - 2  \mathbb{E}_{\pi_{\mu_1, \sigma_1} \otimes \pi_{\mu_2, \sigma_2}} \left[ k_\alpha(X,Y) \right] + \mathbb{E}_{\pi_{\mu_2, \sigma_2} \otimes \pi_{\mu_2, \sigma_2}} \left[ k_\alpha(Y, Y') \right] .
\end{equation}
In this section we prove the following result.

\begin{proposition}
\label{prop:optimal_kernel_gradient_appendix}
For any $\mu_0 \in \rset^d$ and $\sigma >0$, let $\alpha^\star$ be given by 
\begin{equation}
    \textstyle \alpha^\star = \argmax_{\alpha \geq 0} \| \nabla_{\mu_0} \MMD^2_\alpha(\pi_{0, \sigma}, \pi_{\mu_0, \sigma})\| .
\end{equation}
Then, we have that 
\begin{equation}
\label{eq:optimal_alpha_appendix}
    \alpha^\star = \mathrm{ReLU}(\| \mu_0 \|^2/(d+2) - 2 \sigma^2 )^{1/2} . 
\end{equation}
\end{proposition}

Before proving \Cref{prop:optimal_kernel_gradient_appendix}, let us provide some insights on the result. The quantity $\| \nabla_{\mu_0} \MMD^2_\alpha(\pi_{0, \sigma}, \pi_{\mu_0, \sigma})\|$ represents how much the mean of the Gaussian $\pi_{\mu_0, \sigma}$ is displaced when considering a flow on the mean of the Gaussian w.r.t. $\MMD^2_\alpha$. Intuitively, we aim for $\| \nabla_{\mu_0} \MMD^2_\alpha(\pi_{0, \sigma}, \pi_{\mu_0, \sigma})\|$ to be as large as possible as this represents the \emph{maximum displacement possible}. Hence, this justifies our goal of maximizing $\| \nabla_{\mu_0} \MMD^2_\alpha(\pi_{0, \sigma}, \pi_{\mu_0, \sigma})\|$ with respect to the width parameter $\alpha$.

We show that the optimal width $\alpha^\star$ has a closed form given by \eqref{eq:optimal_alpha_appendix}. It is notable that, assuming that $\sigma > 0$ is fixed, this quantity depends on $\| \mu_0 \|$, i.e. how far the modes of the two distributions are. This observation justifies our approach of following an \emph{adaptive} MMD flow at inference time. Finally, we observe that there exists a threshold, i.e. $\|\mu_0\|^2/(d+2) = 2 \sigma^2$ for which lower values of $\| \mu_0 \|$ still yield $\alpha^\star = 0$. This phase transition behavior is also observed in our experiments.

 We define $\mathrm{D}(\alpha, \sigma, \mu_0, \mu_1)$ for any $\alpha, \sigma > 0$ and $\mu_0, \mu_1 \in \rset^d$ given by
\begin{align}
    \textstyle \mathrm{D}(\alpha, \sigma, \mu_0, \mu_1) &= \textstyle \int_{\rset^d \times \rset^d} k_\alpha(x,y) \rmd \pi_{\mu_0, \sigma}(x) \rmd \pi_{\mu_1, \sigma}(y) . 
\end{align}

\begin{proposition}
\label{prop:general_D}
For any $\alpha, \sigma > 0$ and $\mu_0, \mu_1 \in \rset^d$ we have 
\begin{align}
\label{eq:D_full_formula}
    \mathrm{D}(\alpha, \sigma, \mu_0, \mu_1) &=  [\alpha^2 \sigma^2(1/\kappa^2 + 1/\alpha^2)]^{-d/2} \exp[\| \hat{\mu}_0 \|^2 / (2 \kappa^2) + \| \hat{\mu}_1 \|^2 / (2 \kappa^2) \\
    & \qquad \qquad - \langle \hat{\mu}_0 , \hat{\mu}_1 \rangle / \alpha^2 - \| \mu_0 \|^2 / (2 \sigma^2) - \| \mu_1 \|^2 / (2 \sigma^2)] ,
\end{align}
with 
\begin{align}
    \hat{\mu}_1 &= (\alpha^2 \mu_1  + \kappa^2 \mu_0 ) / (\kappa^2 + \alpha^2) , \\
\hat{\mu}_0 &= (\alpha^2 \mu_0 + \kappa^2 \mu_1 ) / (\kappa^2 + \alpha^2) , 
\end{align}
where $\kappa = (1/\sigma^2 + 1/\alpha^2)^{-1/2}$.
\end{proposition}

\begin{proof}
% We denote $k_\alpha$ such that for any $x, y \in \rset^d$, $k_\alpha(x,y) = \alpha^{-d} k_\alpha(x,y)$. Similarly, we denote 
% \begin{equation}
%     \textstyle \mathrm{D}(\alpha, \sigma, \mu_0, \mu_1) &= \textstyle \int_{\rset^d \times \rset^d} k_\alpha(x,y) \rmd \pi_{\mu_0, \sigma}(x) \rmd \pi_{\mu_1, \sigma}(y) . 
% \end{equation}
 In what follows, we start by computing $\mathrm{D}(\alpha, \sigma, \mu_0, \mu_1)$ for any $\alpha, \sigma > 0$ and $\mu_0, \mu_1 \in \rset^d$ given by
\begin{align}
    \textstyle \mathrm{D}(\alpha, \sigma, \mu_0, \mu_1) &= \textstyle \int_{\rset^d \times \rset^d} k_\alpha(x,y) \rmd \pi_{\mu_0, \sigma}(x) \rmd \pi_{\mu_1, \sigma}(y) \\
    & = 1/ (2 \uppi \sigma^2 \alpha)^d \textstyle \int_{\rset^d \times \rset^d} \exp[-\| x - y \|^2 / (2 \alpha^2)] \exp[-\| x - \mu_0\|^2 / (2 \sigma^2)] \exp[-\| y - \mu_1\|^2 / (2 \sigma^2)] \rmd x \rmd y 
    \\
    & = 1/ (2 \uppi \sigma^2 \alpha)^d \textstyle \int_{\rset^d \times \rset^d} \exp[-\| x - y \|^2 / (2 \alpha^2)-\| x - \mu_0\|^2 / (2 \sigma^2)-\| y - \mu_1\|^2 / (2 \sigma^2)] \rmd x \rmd y .
\end{align}
In what follows, we denote $\kappa = (1/ \sigma^2 + 1/\alpha^2)^{-1/2}$. We have 
\begin{align}
\textstyle \mathrm{D}(\alpha, \sigma, \mu_0, \mu_1) &= \textstyle C(\mu_0, \mu_1) \int_{\rset^d \times \rset^d} \exp[-\| x \|^2 / (2 \kappa^2) - \| y \|^2 / (2 \kappa^2) + \langle x, y \rangle / \alpha^2 + \langle x, \mu_0 \rangle / \sigma^2 + \langle y, \mu_1 \rangle / \sigma^2 ] \rmd x \rmd y ,
\end{align}
with $C(\mu_0, \mu_1) = 1/ (2 \uppi \sigma^2 \alpha)^d \exp[-\| \mu_0 \|^2 / (2 \sigma^2) - \| \mu_1 \|^2 / (2 \sigma^2) ]$. In what follows, we denote $\mathrm{P}(x,y)$ the second-order polynomial given by 
\begin{equation}
    \mathrm{P}(x,y) = \| x \|^2 / (2 \kappa^2) + \| y \|^2 / (2 \kappa^2) - \langle x, y \rangle / \alpha^2 - \langle x, \mu_0 \rangle / \sigma^2 - \langle y, \mu_1 \rangle / \sigma^2 . 
\end{equation}
Note that we have 
\begin{equation}
\label{eq:relation_D_P}
     \mathrm{D}(\alpha, \sigma, \mu_0, \mu_1) = \textstyle C(\mu_0, \mu_1) \int_{\rset^d \times \rset^d} \exp[-\mathrm{P}(x,y)] \rmd x \rmd y .
\end{equation}
Next, for any $\hat{\mu}_0, \hat{\mu}_1 \in \rset^d$, we consider $\mathrm{Q}(x,y)$ given by 
\begin{align}
    \mathrm{Q}(x,y) &= \| x - \hat{\mu}_0 \|^2 / (2 \kappa^2) + \| y - \hat{\mu}_1 \|^2 / (2 \kappa^2) - \langle x - \hat{\mu}_0 , y - \hat{\mu}_1 \rangle / \alpha^2 \\
&= \| x \|^2 / (2 \kappa^2)  + \| \hat{\mu}_0 \|^2 / (2 \kappa^2) + \| y \|^2 / (2 \kappa^2) + \| \hat{\mu}_1 \|^2 / (2 \kappa^2) - \langle x, \hat{\mu}_0 \rangle / \kappa^2 - \langle y, \hat{\mu}_1 \rangle / \kappa^2 -  \langle x - \hat{\mu}_0 , y - \hat{\mu}_1 \rangle / \alpha^2 \\
&= \| x \|^2 / (2 \kappa^2)  + \| \hat{\mu}_0 \|^2 / (2 \kappa^2) + \| y \|^2 / (2 \kappa^2) + \| \hat{\mu}_1 \|^2 / (2 \kappa^2) - \langle x, \hat{\mu}_0 \rangle / \kappa^2 - \langle y, \hat{\mu}_1 \rangle / \kappa^2 \\ 
& \qquad -  \langle x, y \rangle / \alpha^2  - \langle \hat{\mu}_0, \hat{\mu}_1 \rangle / \alpha^2 + \langle y , \hat{\mu}_0  \rangle / \alpha^2  + \langle x, \hat{\mu}_1 \rangle / \alpha^2 \\
&= \mathrm{P}(x,y) + \| \hat{\mu}_0 \|^2 / (2 \kappa^2) + \| \hat{\mu}_1 \|^2 / (2 \kappa^2) -  \langle x, \hat{\mu}_0 \rangle / \kappa^2 - \langle y, \hat{\mu}_1 \rangle / \kappa^2 + \langle x, \mu_0 \rangle / \sigma^2 + \langle y, \mu_1 \rangle / \sigma^2 \\
& \qquad  - \langle \hat{\mu}_0, \hat{\mu}_1 \rangle/\alpha^2 + \langle y, \hat{\mu}_0 \rangle / \alpha^2  + \langle x, \hat{\mu}_1 \rangle / \alpha^2 \\
&= \mathrm{P}(x,y) + \| \hat{\mu}_0 \|^2 / (2 \kappa^2) + \| \hat{\mu}_1 \|^2 / (2 \kappa^2)  - \langle \hat{\mu}_0, \hat{\mu}_1 \rangle/\alpha^2 \\
& \qquad + \langle x, \mu_0 / \sigma^2 - \hat{\mu}_0 / \kappa^2 + \hat{\mu}_1 / \alpha^2 \rangle + \langle y, \mu_1 / \sigma^2 - \hat{\mu}_1 / \kappa^2 + \hat{\mu}_0 / \alpha^2 \rangle . 
\end{align}
In what follows, we set $\hat{\mu}_0, \hat{\mu}_1$ such that 
\begin{align}
    \mu_1 / \sigma^2 &= \hat{\mu}_1 / \kappa^2 - \hat{\mu}_0 / \alpha^2 ,\\
    \mu_0 / \sigma^2 &= \hat{\mu}_0 / \kappa^2 - \hat{\mu}_1 / \alpha^2 .
\end{align}
We get that 
\begin{align}
    \hat{\mu}_1 &= (\mu_1 / (\sigma^2 \kappa^2) + \mu_0 / (\sigma^2 \alpha^2)) / (1/\kappa^4 - 1/\alpha^4) , \\
\hat{\mu}_0 &= (\mu_0 / (\sigma^2 \kappa^2) + \mu_1 / (\sigma^2 \alpha^2)) / (1/\kappa^4 - 1/\alpha^4) . 
\end{align}
We have that 
\begin{equation}
\label{eq:identity_alpha_kappa}
    \sigma^2 (1/\kappa^4 - 1/\alpha^4) = \sigma^2(1/\sigma^4 + 2/(\sigma^2 \alpha^2)) = 1/\sigma^2 + 2 / \alpha^2 = 1/\kappa^2 + 1/\alpha^2 . 
\end{equation}
Therefore, we get that 
\begin{align}
    \hat{\mu}_1 &= (\mu_1 / \kappa^2 + \mu_0 / \alpha^2) / (1/\kappa^2 + 1/\alpha^2) , \\
\hat{\mu}_0 &= (\mu_0 /  \kappa^2 + \mu_1 /  \alpha^2) / (1/\kappa^2 + 1/\alpha^2) . 
\end{align}
Finally, we get that 
\begin{align}
    \hat{\mu}_1 &= (\alpha^2 \mu_1  + \kappa^2 \mu_0 ) / (\kappa^2 + \alpha^2) , \\
\hat{\mu}_0 &= (\alpha^2 \mu_0 + \kappa^2 \mu_1 ) / (\kappa^2 + \alpha^2) . 
\end{align}
With this choice, we get that 
\begin{equation}
\label{eq:relation_P_Q}
    \mathrm{P}(x,y) = \mathrm{Q}(x,y) - \| \hat{\mu}_0 \|^2 / (2 \kappa^2) - \| \hat{\mu}_1 \|^2 / (2 \kappa^2) + \langle \hat{\mu}_0, \hat{\mu}_1 \rangle / \alpha^2
\end{equation}
We also have that for any $x, y \in \rset^d$
\begin{equation}
    \mathrm{Q}(x,y) = (1/2) \left( \begin{matrix} x - \hat{\mu}_0 \\ y - \hat{\mu}_1 \end{matrix} \right)^\top \left( \begin{matrix}& \Id/\kappa^2  & -\Id/\alpha^2 \\
    & -\Id/\alpha^2 & \Id/\kappa^2
    \end{matrix} \right)  \left( \begin{matrix} x - \hat{\mu}_0 \\ y - \hat{\mu}_1 \end{matrix} \right)
\end{equation}
Using this result we have that 
\begin{equation}
\label{eq:gaussian_Q_before}
   \textstyle \int_{\rset^d \times \rset^d} \exp[-\mathrm{Q}(x,y)] = (2 \uppi )^d \det(\Sigma^{-1})^{-1/2} ,
\end{equation}
with 
\begin{equation}
    \Sigma^{-1} = \left( \begin{matrix}& \Id/\kappa^2  & -\Id/\alpha^2 \\
    & -\Id/\alpha^2 & \Id/\kappa^2 
    \end{matrix} \right) .
\end{equation}
Using \eqref{eq:identity_alpha_kappa}, we get that 
\begin{equation}
    \det(\Sigma^{-1}) = [(1/\sigma^2)(1/\kappa^2 + 1/\alpha^2)]^d . 
\end{equation}
Combining this result and \eqref{eq:gaussian_Q_before} we get that 
\begin{equation}
\label{eq:gaussian_Q}
    \textstyle \int_{\rset^d \times \rset^d} \exp[-\mathrm{Q}(x,y)] = (2 \uppi )^d [(1/\sigma^2)(1/\kappa^2 + 1/\alpha^2)]^{-d/2} . 
\end{equation}
Combining this result, \eqref{eq:relation_P_Q} and \eqref{eq:relation_D_P} we get that 
\begin{equation}
    \mathrm{D}(\alpha, \sigma, \mu_0, \mu_1) = C(\mu_0, \mu_1) \exp[\| \hat{\mu}_0 \|^2 / (2 \kappa^2) + \| \hat{\mu}_1 \|^2 / (2 \kappa^2) - \langle \hat{\mu}_0 , \hat{\mu}_1 \rangle / \alpha^2] (2 \uppi)^d [(1/\sigma^2)(1/\kappa^2 + 1/\alpha^2)]^{-d/2} . 
\end{equation}
Therefore, we get that 
\begin{align}
\label{eq:D_full_formula_proof}
    \mathrm{D}(\alpha, \sigma, \mu_0, \mu_1) &= [\alpha^2 \sigma^2(1/\kappa^2 + 1/\alpha^2)]^{-d/2} \exp[\| \hat{\mu}_0 \|^2 / (2 \kappa^2) + \| \hat{\mu}_1 \|^2 / (2 \kappa^2) \\
    & \qquad \qquad - \langle \hat{\mu}_0 , \hat{\mu}_1 \rangle / \alpha^2 - \| \mu_0 \|^2 / (2 \sigma^2) - \| \mu_1 \|^2 / (2 \sigma^2)] . 
\end{align}
\end{proof}

We investigate two special cases of \Cref{prop:general_D}. 

First, we show that if $\mu_0 = \mu_1$ then $\mathrm{D}(\alpha, \sigma, \mu_0, \mu_0)$ does not depend on $\mu_0$.

\begin{proposition}
\label{prop:value_D_itself}
For any $\alpha, \sigma >0$ and $\mu_0 \in \rset^d$ we have $\mathrm{D}(\alpha, \sigma, \mu_0, \mu_0) = (\alpha^2 + 2\sigma^2)^{-d/2}$.
\end{proposition}

\begin{proof}
We have that $\hat{\mu}_0 = \hat{\mu}_1 = \mu_1 = \mu_0$ in \Cref{prop:general_D}. In addition, we have that 
\begin{equation}
    (1/2\kappa^2) + (1/2\kappa^2) - 1/\alpha^2 - 1/(2\sigma^2) - 1/(2\sigma^2) = 0 .
\end{equation}
Therefore, we have that 
\begin{equation}
    \exp[\| \hat{\mu}_0 \|^2 / (2 \kappa^2) + \| \hat{\mu}_1 \|^2 / (2 \kappa^2) - \langle \hat{\mu}_0 , \hat{\mu}_1 \rangle / \alpha^2 - \| \mu_0 \|^2 / (2 \sigma^2) - \| \mu_1 \|^2 / (2 \sigma^2)] = 1 , 
\end{equation}
which concludes the proof upon using that $1/\kappa^2 = 1/\alpha^2 + 1/\sigma^2$.
\end{proof}

\Cref{prop:value_D_itself} might seem surprising at first but in fact it simply highlights the fact that when trying to differentiate a Gaussian measure with itself, the result is independent of the location of the Gaussian and only depends on its scale. 
Then, we study the case where $\mu_1 = 0$. 

\begin{proposition}
\label{prop:value_D_joint}
For any $\alpha, \sigma >0$ and $\mu_0 \in \rset^d$ we have 
\begin{equation}
    \mathrm{D}(\alpha, \sigma, \mu_0, 0) =  (\alpha^2 + 2\sigma^2)^{-d/2} \exp[-\| \mu_0 \|^2 /(2(\alpha^2 + 2\sigma^2))].
\end{equation}
\end{proposition}

\begin{proof}
First, we have that 
\begin{equation}
    \hat{\mu}_0 = \alpha^2/(\kappa^2+\alpha^2)^2 \mu_0 , \qquad \hat{\mu}_1 = \kappa^2/(\kappa^2+\alpha^2)^2 \mu_0 . 
\end{equation}
Therefore, we get that 
\begin{equation}
    \mathrm{D}(\alpha, \sigma, \mu_0, 0) = [\sigma^2(1/\kappa^2 + 1/\alpha^2)]^{d/2} \exp[(1/2)\{ (\alpha^4/\kappa^2 - \kappa^2)/(\kappa^2 + \alpha^2) - 1 /\sigma^2\} \| \mu_0 \|^2]
\end{equation}
Using \eqref{eq:identity_alpha_kappa} we get that 
\begin{equation}
    \alpha^4/\kappa^2 - \kappa^2 = \alpha^2 (\alpha^2 + \kappa^2) /\sigma^2 .
\end{equation}
Therefore, we get that 
\begin{equation}
    (\alpha^4/\kappa^2 - \kappa^2)/(\kappa^2 + \alpha^2) - 1 /\sigma^2 = (\alpha^2/(\alpha^2 + \kappa^2) - 1)/\sigma^2 = -1 /(\alpha^2(1 + 2\sigma^2/\alpha^2)) ,
\end{equation}
which concludes the proof.
\end{proof}

Using \Cref{prop:value_D_itself}, \Cref{prop:value_D_joint} and definition~\eqref{eq:mmd_squared_appendix}, we have the following result.

\begin{proposition}
\label{prop:mmd_mmd_derivative}
For any $\alpha, \sigma >0$ and $\mu_0 \in \rset^d$ we have
\begin{equation}
    \MMD^2(\pi_{0, \sigma}, \pi_{\mu_0, \sigma}) = 2(\alpha^2 + 2\sigma^2)^{-d/2} (1 - \exp[-\| \mu_0 \|^2 /(2(\alpha^2 + 2\sigma^2))]) . 
\end{equation}
In addition, we have 
\begin{equation}
    \nabla_{\mu_0} \MMD^2(\pi_{0, \sigma}, \pi_{\mu_0, \sigma}) = -2 (\alpha^2 + 2\sigma^2)^{-d/2-1} \exp[-\| \mu_0 \|^2 /(2(\alpha^2 + 2\sigma^2))] \mu_0 . 
\end{equation}
\end{proposition}

Finally, we have the following proposition.

\begin{proposition}
For any $\mu_0 \in \rset^d$ and $\sigma >0$ let $\alpha^\star$ be given by 
\begin{equation}
    \textstyle \alpha^\star = \argmax_{\alpha \geq 0} \| \nabla_{\mu_0} \MMD^2(\pi_{0, \sigma}, \pi_{\mu_0, \sigma})\| .
\end{equation}
Then, we have that 
\begin{equation}
    \alpha^\star = \mathrm{ReLU}(\| \mu_0 \|^2/(d+2) - 2 \sigma^2 )^{1/2} . 
\end{equation}
\end{proposition}

\begin{proof}
Let $\sigma > 0$ and $\mu_0 \in \rset^d$.
First, using \Cref{prop:mmd_mmd_derivative}, we have that for 
\begin{equation}
    \| \nabla_{\mu_0} \MMD^2(\pi_{0, \sigma}, \pi_{\mu_0, \sigma}) \|^2 = 4 \alpha^{2d} \| \mu_0 \|^2 (\alpha^2 + 2\sigma^2)^{-d-2}\exp[-\| \mu_0 \|^2 / (\alpha^2 + 2 \sigma^2)] .
\end{equation}
Next, we study the function $\mathrm{f}: \ [0, t_0] \to \rset$ given for any $t  \in [0, t_0]$ by
\begin{equation}
    \mathrm{f}(t) = t^{d+2} \exp[-t\| \mu_0 \|^2] ,
\end{equation}
with $t_0 = 1/(2\sigma^2)$. We have that 
\begin{equation}
    \mathrm{f}'(t) = t^{d+1} \exp[-t \| \mu_0 \|^2] ((d+2) - \| \mu_0 \|^2 t) . 
\end{equation}
We then consider two cases. First, if $t_0 \leq (d+2) / \| \mu_0 \|^2$, i.e. $\sigma^2 \leq \| \mu_0 \|^2 / (2(d+2))$, then $\mathrm{f}$ is increasing on $[0,t_0]$ and we have that $f$ is maximum if $t = t_0$. Hence, if $\sigma^2 \leq \| \mu_0 \|^2 / (2(d+2))$, we have that $\alpha^\star = 0$. Second, if  $t_0 \leq (d+2) / \| \mu_0 \|^2$, i.e. $\sigma^2 \leq \| \mu_0 \|^2 / (2(d+2))$ then $\mathrm{f}$ is increasing on $[0,t^\star]$, non-increasing on $[t^\star, t_0]$ with $t^\star = (d+2) / \| \mu_0 \|^2$ and we have that $f$ is maximum if $t = t^\star$. Hence, if $\sigma^2 \geq \| \mu_0 \|^2 / (2(d+2))$, we have that $\alpha^\star = (\| \mu_0 \|^2 / (d+2) - 2 \sigma^2)^{1/2}$, which concludes the proof.
\end{proof}

\subsection{Phase transition behaviour}

\begin{figure}[t]
    \centering
    \includegraphics[scale=0.27]{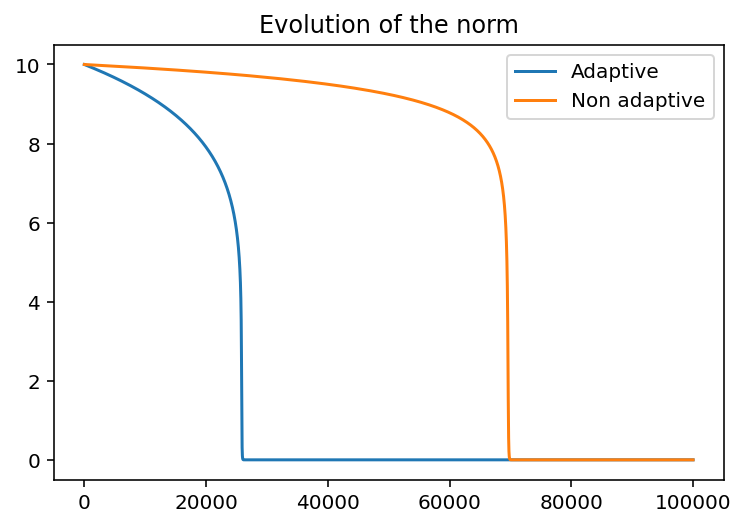}
    \hfill
    \includegraphics[scale=0.27]{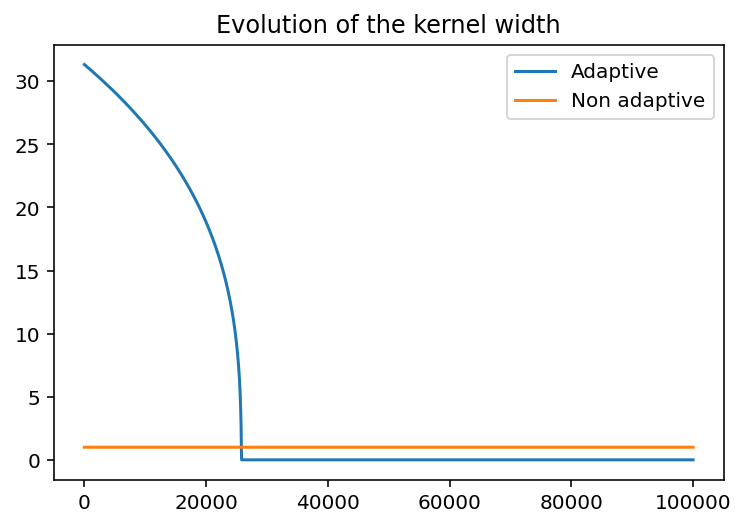}
    \caption{Evolution of the norm of the mean $\mu_t$ of the Gaussian distribution $\pi_{\mu_t, \sigma}$ according to a gradient flow on the mean $\mu_t$ w.r.t. $\MMD_{\alpha_t}$. In the \emph{adaptive} case $\alpha_t$ is given by \Cref{prop:optimal_kernel_gradient} while in the \emph{non adaptive} case, $\alpha_t = \alpha_0 = 1$. In our experiment we consider $d=1$ and $\sigma =1$, for illustration purposes.}
\label{fig:grad_flow_gaussian}
\end{figure}

\section{Image generation samples}
\label{app:all_samples}

\subsection{CIFAR10 samples}
\label{app:samples}

Samples from $\DMMD$ with NFE=100 and NFE=250 are given in Figure~\ref{fig:samples_nfe100_vs_250}. Samples from $\DMMD$ with NFE=100 and from $a$-$\DMMD$ with NFE=50 are given in Figure~\ref{fig:samples_dmmd_vs_admmd}.

\begin{figure}
    \centering
    \includegraphics{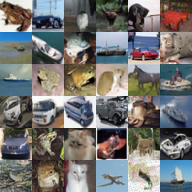}
    \includegraphics{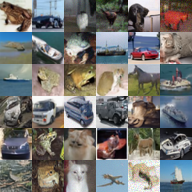}
    \caption{CIFAR-10 samples from $\DMMD$ with NFE=250 on the left and with NFE=100 on the right}
    \label{fig:samples_nfe100_vs_250}
\end{figure}

\begin{figure}
    \centering
    \includegraphics{figures/cifar10_samples_nfe_100.png}
    \includegraphics{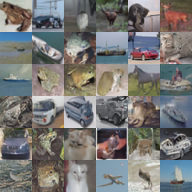}
    \caption{CIFAR-10 samples from $\DMMD$ with NFE=100 on the left and samples from the $a$-$\DMMD$-$e$ with NFE=50 on the right}
    \label{fig:samples_dmmd_vs_admmd}
\end{figure}

\subsection{Additional datasets samples}
\label{app:additional_datasets_samples}

Samples for MNIST are given in Figure~\ref{fig:samples_mnist}, for CELEB-A (64x64) are given in Figure~\ref{fig:samples_celeba} and for LSUN Church (64x64) are given in Figure~\ref{fig:samples_lsun_church}.

\begin{figure}
    \centering
    \includegraphics{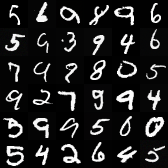}
    \caption{$\DMMD$ samples for MNIST.}
    \label{fig:samples_mnist}
\end{figure}

\begin{figure}
    \centering
    \includegraphics{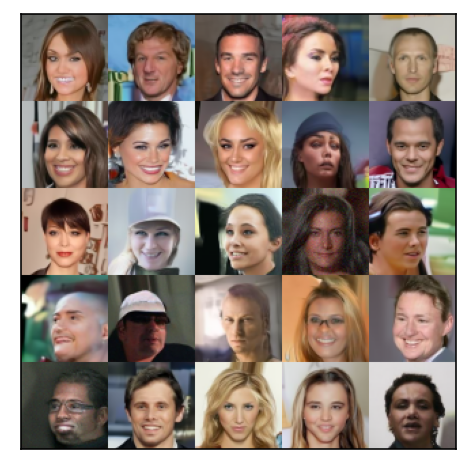}
    \caption{$\DMMD$ samples for CELEB-A (64x64).}
    \label{fig:samples_celeba}
\end{figure}

\begin{figure}
    \centering
    \includegraphics{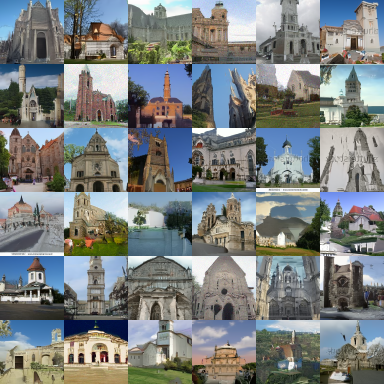}
    \caption{$\DMMD$ samples for LSUN Church (64x64).}
    \label{fig:samples_lsun_church}
\end{figure}

\end{document}